\documentclass{article}

\usepackage{microtype}
\usepackage{graphicx}
\usepackage{subfigure}
\usepackage{booktabs} 

\usepackage{hyperref}

\usepackage[accepted]{icml2024}

\usepackage{amsmath}
\usepackage{amssymb}
\usepackage{mathtools}
\usepackage{amsthm}

\usepackage[capitalize,noabbrev]{cleveref}

\theoremstyle{plain}
\newtheorem{theorem}{Theorem}[section]
\newtheorem{proposition}[theorem]{Proposition}
\newtheorem{lemma}[theorem]{Lemma}
\newtheorem{corollary}[theorem]{Corollary}
\theoremstyle{definition}
\newtheorem{definition}[theorem]{Definition}

\theoremstyle{remark}
\newtheorem{remark}[theorem]{Remark}

\usepackage[textsize=tiny]{todonotes}

\usepackage{tabularx}
\newcommand{\comp}{\bullet }

\icmltitlerunning{Vocabulary for Universal Approximation}

\begin{document}

\twocolumn[
\icmltitle{Vocabulary for Universal Approximation: \\A Linguistic Perspective of Mapping Compositions}




\begin{icmlauthorlist}

\icmlauthor{Yongqiang Cai}{bnu}

\end{icmlauthorlist}



%

\icmlaffiliation{bnu}{School of Mathematical Sciences, Laboratory of Mathematics and Complex Systems, MOE, Beijing Normal University, Beijing, 100875, China}
\icmlcorrespondingauthor{Yongqiang Cai}{caiyq.math@bnu.edu.cn}

\icmlkeywords{Machine Learning, ICML}

\vskip 0.3in
]



\printAffiliationsAndNotice{}  

\begin{abstract}
    In recent years, deep learning-based sequence modelings, such as language models, have received much attention and success, which pushes researchers to explore the possibility of transforming non-sequential problems into a sequential form. Following this thought, deep neural networks can be represented as composite functions of a sequence of mappings, linear or nonlinear, where each composition can be viewed as a \emph{word}. However, the weights of linear mappings are undetermined and hence require an infinite number of words. In this article, we investigate the finite case and constructively prove the existence of a finite \emph{vocabulary} $V=\{\phi_i: \mathbb{R}^d \to \mathbb{R}^d ~|~ i=1,...,n\}$ with $n=O(d^2)$ for the universal approximation. That is, for any continuous mapping $f: \mathbb{R}^d \to \mathbb{R}^d$, compact domain $\Omega$ and $\varepsilon>0$, there is a sequence of mappings $\phi_{i_1}, ..., \phi_{i_m} \in V, m \in \mathbb{Z}_+$, such that the composition $\phi_{i_m} \circ ... \circ \phi_{i_1} $ approximates $f$ on $\Omega$ with an error less than $\varepsilon$. Our results demonstrate an unusual approximation power of mapping compositions and motivate a novel compositional model for regular languages.
\end{abstract}

\section{Introduction}

Cognitive psychologists and linguisticians have long recognized the importance of languages \citep{pinker2003language}, which has been further highlighted by the popularity of language models such as BERT \citep{Devlin2018BERT} and GPT \citep{Brown2020Language}. These models, based on RNNs or Transformers, have revolutionized natural language processing by transforming it into a sequence learning problem. They can handle the long-term dependencies in text and generate coherent text based on the previous content, making them invaluable tools in language understanding and generation \citep{Vaswani2017Attention}.
The success of these models has also led to a new approach to solving non-sequential problems by transforming them into sequential ones. For instance, image processing can be turned into a sequence learning problem by segmenting an image into small blocks, arranging them in a certain order, and then processing the resulting sequence using sequence learning algorithms to achieve image recognition \citep{Dosovitskiy2020Image}.
The use of sequence learning algorithms has also been extended to reinforcement learning \citep{chen2021decision}, such as the decision transformer which outputs the optimal actions by leveraging a causally
masked transformer and exceeds the state-of-the-art performance.

Sequence modeling has opened up new possibilities for solving a wide range of problems, and this trend seems to hold in the field of theoretical research. 
As is well known, artificial neural networks have universal approximation capabilities, and wide or deep feedforward networks can approximate continuous functions on a compact domain arbitrarily well \citep{Cybenko1989Approximation, Hornik1989Multilayer, Leshno1993Multilayer}. 
However, in practical applications such as AlphaFold \citep{jumper2021highly}, BERT \citep{Devlin2018BERT} and GPT \citep{Brown2020Language}, the residual network structures \citep{he2016deep,he2016identity} are more preferred than the feedforward structures.
It is observed that residual networks (ResNets) are forward Euler discretizations of dynamical systems \citep{Weinan2017A, Sander2022Residual}, and this relationship has spawned a series of dynamical system-based neural network structures such as the neural ODE \citep{Chen2018Neural}. 
The dynamical system-based neural network structures are expected to play an important role in various fields.

Notably, both the language models and the dynamical systems are linked to time series modeling and have been effectively applied to non-sequential problems. This observation naturally leads us to question: 
\emph{
is there an intrinsic relationship between their individual successes?
}
This article aims to ponder upon the question. Through a comparative study, we obtain some initial results from the perspective of universal approximation. Specifically, we demonstrate that there exists a finite set of mappings, referred to as the \emph{vocabulary} $V$, chosen as flow maps of some autonomous dynamical system $x'(t) = f(x(t))$, such that any continuous mapping can be approximated by composing a sequence of mappings in the vocabulary $V$. 
This bears a resemblance to the way complex information is conveyed in language through constructing phrases, sentences, and ultimately paragraphs and compositions. Table \ref{tab:main} provides an intuitive representation of this similarity.

\begin{table*}[!htbp]\label{tab:main}
    \caption{Comparison of languages and dynamical systems in dimension $d$}
    {\centering
    \begin{tabularx}{\textwidth}{lXX}
        \toprule
        &  \textbf{English} & \textbf{Flow map of dynamical systems} $^\ddagger$  \\
        \midrule
        Vocabulary 
        &  $\sim$140,000$^\dagger$
        & $O(d^2)$
        \\
        \midrule
        Word 
        & I, you, am, is, are, apple, banana, car, buy, do, have, blue, red,... 
        &  $\phi^\tau_{\pm e_1}$,  $\phi^\tau_{\pm e_2}$, ...,  $\phi^\tau_{\pm e_d}$, $\phi^\tau_{\pm E_{11}x}$, $\phi^\tau_{\pm E_{12}x}$, $\phi^\tau_{\pm E_{21}x}$, ..., $\phi^\tau_{\pm E_{dd}x}$, $\phi^\tau_{\pm\text{ReLU}(x)}$, ...
        \\
        \midrule
        Phrase 
        & A big deal, easier said than done, time waits for no man, ...
        & $\phi^\tau_{e_1}\comp\phi^\tau_{-e_2}$, $\phi^\tau_{e_1}\comp \phi^\tau_{E_{11}x}\comp \phi^\tau_{\text{ReLU}(x)}$, ...
        \\
        \midrule
        Sentence  
        &  It was the best of times, it was the worst of times, it was the age of wisdom, it was the age of foolishness, ...
        & $\phi^\tau_{e_3} \comp \phi^\tau_{\text{ReLU}(x)} \comp \phi^\tau_{-E_{21}x} \comp \phi^\tau_{\text{ReLU}(x)} \comp \phi^\tau_{E_{23}x} \comp$ $\phi^\tau_{-e_2} \comp \phi^\tau_{\text{ReLU}(x)} \comp \phi^\tau_{E_{11}x} \comp \phi^\tau_{e_1} \comp \dots$
        \\
        \bottomrule
    \end{tabularx}%
    }\\
    $\dagger$ The number of words, phrases, and meanings in Cambridge Advanced Learner's Dictionary.\\
    $\ddagger$ Notations are provided in Section \ref{sec:main}.
\end{table*}%

\subsection{Contributions}

\begin{enumerate}
    \item We proved that it is possible to achieve the universal approximation property by composing a sequence of mappings in a finite set $V$. (Theorem \ref{th:main-OP} and Corollary \ref{th:main-C}).
    
    \item Our proof is constructive as we designed such a $V$ that contains a finite number of flow maps of dynamical systems. (Theorem \ref{th:main_V})

    \item {We observed a similarity between composite continuous mappings and words in languages (Table \ref{tab:main}). Furthermore, we proved that embedding words as continuous mappings could serve as a model of regular languages (Theorem \ref{th:L_flowable}).
    }
\end{enumerate}

\subsection{Related works}

\textbf{Universal approximation.} 
The approximation properties of neural networks have been extensively studied, with previous studies focusing on the approximation properties of network structures such as feedforward neural networks \citep{Cybenko1989Approximation, Hornik1989Multilayer, Leshno1993Multilayer} and residual networks \citep{he2016deep,he2016identity}. In these networks, the structure is fixed and the weights are adjusted to approximate target functions. Although this paper also considers universal approximation properties, we use a completely different way. We use a finite set of mappings, and the universal approximation is achieved by composing sequences of these mappings. 
The length of the mapping sequence is variable, which is similar to networks with a fixed width and variable depth \citep{Lu2017Expressive, Johnson2019Deep, Kidger2020Universal, Park2021Minimum, Beise2020Expressiveness, cai2022achieve,li2023minimum}. However, in our study, we do not consider learnable weights; instead, we consider the composition sequence, which is different from previous research.

\textbf{Residual network, neural ODE, and control theory.} 
The word mapping constructed in this paper is partially based on the numerical discretization of dynamical systems and therefore has a relationship with residual networks and neural ODEs. Residual networks \citep{he2016deep,he2016identity} are currently one of the most popular network structures and can be viewed as a forward Euler discretization of neural ODEs \citep{Chen2018Neural}. 
Recently, \citet{Li2022Deep} and \citet{Tabuada2022Universal} studied the approximation properties of neural ODEs. Their basic idea is employing controllability results in control theory to construct source terms that approximate a given finite number of input-output pairs, thus obtaining the approximation properties of functions in the $L^p$ norm or continuous norm sense. Additionally, \citet{Duan2022Vanilla} proposed an operator splitting format that discretizes neural ODEs into leaky-ReLU fully connected networks. Partially inspired by Duan et al.'s construction, we designed a special splitting method to finish one part of our construction. 

It's worth noting that all neural networks mentioned above can be represented as compositions of mapping sequences. However, the networks involve an infinite number of mappings, which is different from our construction which only requires a finite number of mappings. 

\textbf{Compositionality.} 
Our results demonstrate that the composition is a powerful operator that allows us to achieve the universal approximation property on compact domains by using a finite number of mappings. This is a little similar to the concept of compositionality in linguistics, especially in the Montagovian framing \citep{Montague1970Universal, Kracht2012Compositionality}, which is the idea that a finite vocabulary of basic elements can be combined via a grammar to express an infinite range of meanings. Recently, researchers have explored the capabilities of neural models to acquire compositionality while learning from data \citep{Dankers2022Paradox, Valvoda2022Benchmarking}. However, they focused on algebraic relations rather than approximations. It's interesting to think whether these studies and ours can be connected.

\textbf{Word embeding.}
The finite mapping vocabulary might be related to the word embedding in natural language processing. The most basic model involves embedding words as vectors and then summing these word vectors to obtain the sentence vector \citep{Mikolov2013Efficient}.
However, the summation operator is commutative, and thus vector embedding models fail to capture any notion of word order. To address this limitation, \citet{Rudolph2010Compositional} proposed modeling words as matrices rather than vectors and composing sentence embeddings through matrix multiplication instead of addition.
For recent advancements in this direction, we refer to \citet{Mai2018CBOW, Asaadi2023Compositional}. 
To the best of our knowledge, prior research in this domain has not delved into the approximation properties. Leveraging the techniques presented in this paper, we can readily establish the existence of a finite vocabulary for both vector embedding and matrix embedding (see Appendix~\ref{sec:V_linear}).
Furthermore, embedding words as matrices offers a compositional model for regular languages \cite{Rudolph2010Compositional}, which can be generalized to continuous mapping embeddings (see Section~\ref{sec:flow_grammar}). 
It is important to note that vector space and matrix space are finite-dimensional, while the continuous function space is infinite-dimensional. This suggests that embedding words as nonlinear mappings could enhance the expressiveness of sentences. However, there is limited exploration in this direction.

For embedding words as functions, there is a related work named Word2Fun \cite{Wang2021Word2Fun} which aims to model time in word representation for some diachronic tasks. Note that the Word2Fun model does not involve the mapping compositions and hence is very different from the setting in this paper.

\subsection{Outline}

We state the main result for universal approximation in Section \ref{sec:main}, which includes notations, main theorems, and ideas for construction and proof. Before providing the detailed construction in Section \ref{sec:construction}, we add a Section \ref{sec:warmup} to introduce flow maps and the techniques we used. The linguistic implication of our results is given in Section \ref{sec:flow_grammar}. Finally, in Section \ref{sec:conclusion} we discuss the result of this paper. All formal proof of the theorems is provided in the Appendix.

\section{Notations and main results}
\label{sec:main}

\subsection{Preliminaries}

The statement and the proof of our main results contain some concepts in mathematics. Here we provide a brief introduction for them, which is enough to understand most parts of this paper. 

One concept is the orientation-preserving (OP) diffeomorphisms of $\mathbb{R}^d$. A differentiable map $f: \mathbb{R}^d \to \mathbb{R}^d$ is called a diffeomorphism if it is a bijection and its inverse $f^{-1}$ is differentiable as well. In addition, a diffeomorphism $f$ of $\mathbb{R}^d$ is called orientation-preserving if the Jacobian of $f$ is positive everywhere. A simple example of OP diffeomorphisms is the linear map $f: x \to Px$ where $x\in\mathbb{R}^d$ and $P$ is a square matrix with positive determinant. 

Another concept is the flow map of dynamical systems. Here the dynamical system is characterized by the following ordinary differential equation (ODE) in dimension $d$,
\begin{align}\label{eq:ODE_general}
    \left\{
    \begin{aligned}
    &\dot{x}(t) = v(x(t),t), t\in(0,\tau),\\
    &x(0)=x_0 \in \mathbb{R}^d,
    \end{aligned}
    \right .
\end{align}
where $v:\mathbb{R}^d\to\mathbb{R}^d$ is the velocity field and $x_0$ is the initial value. When the field $v$ satisfies some conditions, such as Lipschitz continuous, the ODE \eqref{eq:ODE_general} has a unique solution $x(t), t \in [0,\tau]$. Then the map from the initial state $x_0$ to $x(\tau)$, the state of the system after time $\tau$, is called the flow map and denoted by $\phi_{v(x,t)}^\tau(x_0)$, where $x_0$ is allowed to vary. A basic property is that the flow maps are naturally orientation-preserving. For example, let $A$ be a square matrix and $v(x,t) = Ax$, then the flow map $\phi_{v(x,t)}^\tau(x_0)$ is a linear map $\phi_{Ax}^\tau(x_0) = e^{A\tau} x_0$, where $e^{A\tau}$ is the matrix exponential of $A\tau$. A deeper introduction and understanding of flows and dynamical systems can be found in Chapter 1 of \citet{Arrowsmith1990introduction}.

\subsection{Notations}

For a (vector valued) function class $\mathcal{F}$, the \emph{vocabulary} $V$ is defined as a finite subset of $\mathcal{F}$, \emph{i.e.},
\begin{align}
    V = \{\phi_1, \phi_2,..., \phi_n\} \subset \mathcal{F}, \quad n \in \mathbb{Z}_+.
\end{align}
Each $\phi_i \in V$ is called a \emph{word}. We will consider a sequence of functions, $\phi_{i_1},\phi_{i_2},...,\phi_{i_m} \in V$, and their composition, called as a \emph{sentence}, to generate the hypothesis function space,
\begin{align} \label{eq:H_v}
    \mathcal{H}_{V} = \{\phi_{i_1} \comp ... \comp\phi_{i_m} | \phi_{i_1},...,\phi_{i_m}  \in V, m \in \mathbb{Z}_+\}.
\end{align}
Particularly, some (short) sentences are called \emph{phrases} for some purpose. Here the operator $\comp$ is defined as function composition from left to right, which aligns the composition order to the writing order, \emph{i.e.}
\begin{align*}
    \phi_{i_1} \comp \phi_{i_2} \comp ... \comp\phi_{i_m} 
    &= \phi_{i_m} \circ ... \circ \phi_{i_2} \circ \phi_{i_1}\\
    &= \phi_{i_m}( ... (\phi_{i_2} (\phi_{i_1}(\cdot)))...).
\end{align*}
In additional, we use $\phi^{\comp m},$ to denote the mapping that composites $\phi$ $m$ times. 

In this paper, we consider two function classes: (1) $C(\mathbb{R}^d,\mathbb{R}^{d})$, continuous functions from $\mathbb{R}^d$ to $\mathbb{R}^{d}$, (2) $\text{Diff}_0(\mathbb{R}^d)$, OP diffeomorphisms of $\mathbb{R}^d$, whose closure in $C(\mathbb{R}^d,\mathbb{R}^{d})$ is denoted as $\overline{\text{Diff}_0(\mathbb{R}^d)}$. Particularly, we will restrict the functions on a compact domain $\Omega \subset \mathbb{R}^d$ and define the universal approximation property as below.
\begin{definition}[Universal approximation property, UAP] 
    For the compact domain $\Omega$ in dimension $d$, the target function space $\mathcal{F}$ and the hypothesis space $\mathcal{H}$, we say
    \begin{enumerate}
        \item $\mathcal{H}$ has $C$-UAP for $\mathcal{F}$, if for any $f \in \mathcal{F}$ and $\varepsilon>0$, there is a function $h \in \mathcal{H}$ such that $$\|f(x)-h(x)\| < \varepsilon, \quad \forall x \in \Omega.$$
        
        \item $\mathcal{H}$ has $L^p$-UAP, $p \in [1,+\infty)$, for $\mathcal{F}$, if for any $f \in \mathcal{F}$ and $\varepsilon>0$, there is a function $h \in \mathcal{H}$ such that 
        $$ \|f-h\|_{L^p(\Omega)} = 
        \Big(\int_\Omega \|f(x)-h(x)\|^p dx\Big)^{1/p} < \varepsilon.$$
    \end{enumerate}
\end{definition}

Remark that here we use the term $C$-UAP instead of $L^\infty$-UAP for two reasons: (1) $C$ and $L^p$ represent both norms and function spaces; The norms $L^{\infty}$ and $C$ have subtle differences and the space $L^{\infty}(\Omega, \mathbb R^d)$ is significantly different from $C(\Omega, \mathbb R^d)$. (2) $L^{\infty}$-UAP is stronger than $L^p$-UAP; Not using $L^{\infty}$ is to avoid having to specifically emphasize that $p$ does not include $\infty$ when referring to $L^p$.

\subsection{Main theorem}

Our main result is Theorem \ref{th:main-OP} and its Corollary \ref{th:main-C}  which show the existence of a finite function vocabulary $V$ for the universal approximation property.

\begin{theorem}\label{th:main-OP}
Let $\Omega \subset \mathbb{R}^d$ be a compact domain. Then, there is a finite set $V \subset \overline{\text{Diff}_0(\mathbb{R}^d)}$ such that the hypothesis space $\mathcal{H}_V$ in Eq.~\eqref{eq:H_v} has $C$-UAP for $\text{Diff}_0(\mathbb{R}^d)$.
\end{theorem}

\begin{corollary}\label{th:main-C}
    Let $\Omega \subset \mathbb{R}^d$ be a compact domain, $d\ge 2$ and $p \in [1,+\infty)$. Then, there is a finite set $V \subset C(\mathbb{R}^d,\mathbb{R}^{d})$ such that the hypothesis space $\mathcal{H}_V$ in Eq.~\eqref{eq:H_v} has $L^p$-UAP for $C(\mathbb{R}^d,\mathbb{R}^{d})$.
\end{corollary}
The Corollary \ref{th:main-C} is based on the fact that OP diffeomorphisms can approximate continuous functions under the $L^p$ norm provided the dimension is larger than two \citep{Brenier2003Approximation} . Next, we only need to prove Theorem \ref{th:main-OP}.

\begin{remark}
    We are considering functions to have the same dimension of the input and output, for simplicity. Our results can be directly extended to the case of different input and output dimensions. In fact, for $ f \in C(\mathbb{R}^{d_x},\mathbb{R}^{d_y})$, one can lift it as a function $ \tilde f \in C(\mathbb{R}^{d},\mathbb{R}^{d})$ with some $d \ge \max(d_x,d_y)$. For example, let $f = A_{in} \comp \tilde{f} \comp A_{out} $ where $A_{in}\in C(\mathbb{R}^{d_x},\mathbb{R}^{d})$ and $A_{out}\in C(\mathbb{R}^{d},\mathbb{R}^{d_y})$ are two fixed affine mappings. 
\end{remark}

\subsection{Sketch of the proof }

Our proof for Theorem \ref{th:main-OP} is constructive, by considering the flow maps of ODEs. In particular, our construction will use the following two classes of candidate flow maps in dimension $d$,
\begin{align*}
    H_1 &= 
    \big\{ \phi^\tau_{Ax+b} ~|~ A \in \mathbb{R}^{d \times d}, b \in \mathbb{R}^{d},\tau \ge 0 \big\}\\
    &\equiv 
    \big\{ \phi : x \to e^{\tilde A} x+\tilde b ~|~ \tilde A \in \mathbb{R}^{d \times d},\tilde b \in \mathbb{R}^{d} \big\},\\
    H_2 &= 
    \big\{ \phi^\tau_{\Sigma_{\boldsymbol{\alpha},\boldsymbol{\beta}}(x)}  ~|~ \boldsymbol{\alpha},\boldsymbol{\beta} \in \mathbb{R}^d, \tau \ge 0 \big\} \\
    &\equiv 
    \big\{ \phi : x \to \Sigma_{\boldsymbol{\tilde \alpha},\boldsymbol{\tilde \beta}} (x) ~|~ \boldsymbol{\tilde \alpha},\boldsymbol{\tilde \beta} \in (0,+\infty)^d \big\}
    ,
\end{align*}
where $\Sigma_{\boldsymbol{\alpha},\boldsymbol{\beta}}$ is the generalized leaky-ReLU functions defined as below. We say $H_1$ the affine flows and $H_2$ the leaky-ReLU flows.
\begin{definition}[Generalized leaky-ReLU] Define the generalized leaky-ReLU function as $\sigma_{\alpha,\beta}: \mathbb{R}\to\mathbb{R}$ and $\Sigma_{\boldsymbol{\alpha},\boldsymbol{\beta}}: \mathbb{R}^d\to\mathbb{R}^d$, with $\alpha,\beta \in \mathbb{R}$, $\boldsymbol{\alpha} = (\alpha_1,...,\alpha_d) \in \mathbb{R}^d$,  $\boldsymbol{\beta}=(\beta_1,...,\beta_d) \in \mathbb{R}^d$,
\begin{align}
    \sigma_{\alpha,\beta}(x) = 
    \begin{cases}
        \alpha x , & x <0,\\
        \beta x, &x \ge 0,
    \end{cases}
\end{align}
\begin{align}
    \Sigma_{\boldsymbol{\alpha},\boldsymbol{\beta}}(x) = 
    \big(
        \sigma_{\alpha_1,\beta_1}(x_1),...,
    \sigma_{\alpha_d,\beta_d}(x_d) \big).
\end{align}
\end{definition}
Generalized leaky-ReLU functions are piecewise linear functions. 
Using this notation, the traditional ReLU and leaky-ReLU functions are $\text{ReLU}(x) \equiv \sigma_0(x) \equiv \sigma_{0,1}(x)$ and $\sigma_\alpha(x) \equiv \sigma_{\alpha,1}(x)$ with $\alpha \in (0,1)$, respectively. For vector input $x$, we use $\sigma_{\alpha,\beta}$ as an equivilant notation of $\Sigma_{\alpha \boldsymbol{1}, \beta \boldsymbol{1}}$.

We will show that the following set $V$ meets our requirement for universal approximations,
\begin{align}
    V = \big\{\phi_{\pm e_i}^\tau, \phi_{\pm E_{ij}x}^\tau, 
    \phi_{\pm \Sigma_{e_i,0}(x)}^\tau,
    \phi_{\pm \Sigma_{0,e_i}(x)}^\tau  
    ~|~ \nonumber \\
    i,j \in \{1,2,...,d\}, 
    \tau \in \{1, \sqrt{2}\} \big\},  \label{eq:V}
\end{align}
where $e_i \in \mathbb{R}^{d}$ is the $i$-th unit coordinate vector, $E_{ij}$ is the $d\times d$ matrix that has zeros in all entries except for a 1 at the index $(i,j)$. Obviously, $V \subset \overline{\text{Diff}_0(\mathbb{R}^d)}$ is a finite set with $O(d^2)$ functions.
\begin{theorem}\label{th:main_V}
    Let $\Psi \in \text{Diff}_0(\Omega)$ be an orientation preserving diffeomorphism, $\Omega$ be a compact domain $\Omega \subset \mathbb{R}^d$. Then, for any $\varepsilon>0$, there is a sequence of flow maps, $\phi_1, \phi_2, ...,\phi_n \in V, n \in \mathbb{Z}_+$, such that
    \begin{align}
        \|\Psi(x)-(\phi_1\comp \phi_2\comp ...\comp\phi_n)(x)\| \le \varepsilon, \quad
        \forall x \in \Omega.
    \end{align}
\end{theorem}
Theorem \ref{th:main_V} provides a constructive proof for Theorem \ref{th:main-C}. The proof of Theorem \ref{th:main_V} can be separated into the following two parts. 
\begin{itemize}

    \item[] \textbf{Part 1}: Approximate each flow map in $H_1$ and $H_2$ by composing a sequence of flow maps in $V$.
    
    \item[] \textbf{Part 2}: Approximate $\Psi \in \text{Diff}_0(\mathbb{R}^d)$ by composing a sequence of flow maps in $H_1 \cup H_2$. Particularly, we approximate $\Psi$ by $g_L$ of the form
    \begin{align}\label{eq:g_L}
        g_L = h_{0} \comp h_{1}^* \comp h_{1} \comp h_{2}^* \comp h_{2} \comp 
        ... \comp h_{L}^* \comp h_{L}, 
    \end{align}
    where $h_{i} \in H_1, h_{i}^* \in H_2, L \in \mathbb{Z}_+.$

\end{itemize}

The validation of such constructed $V$ is technical and will be proved in Section \ref{sec:warmup} and Section \ref{sec:construction}. Here we only explain the main ideas. First of all, we note that to approximate a composition map $T$, we only need to approximate each component in $T$, which is detailed in the following Lemma~\ref{th:composition_approximation}.
\begin{lemma}\label{th:composition_approximation}
  Let map $T = F_1 \comp ... \comp F_n$ be a composition of $n$ continuous functions $F_i$ defined on an open domain $D_i$, and let $\mathcal{F}$ be a continuous function class that can uniformly approximate \footnote{`Uniformly approximate' means the approximation under the uniform/continuous norm.} each $F_i$ on any compact domain $\mathcal{K}_i \subset D_i$. Then, for any compact domain $\mathcal{K} \subset D_1$ and $\varepsilon >0$, there are $n$ functions $\tilde F_1, ..., \tilde F_n$ in $\mathcal{F}$ such that
  \begin{align}
          \|T(x) - \tilde F_1 \comp ... \comp \tilde F_n (x)\|
          \le \varepsilon,
          \quad
          \forall x \in \mathcal{K}.
    \end{align}
\end{lemma}

For Part 1, the validation involves three techniques in math: the Lie product formula \citep{Hall2015Matrix}, the splitting method \citep{holden2010splitting} and the Kronecker's theorem \citep{Apostol1990Kronecker}. We take $\phi_{b}^1 \in H_1, b= \sum_{i=1}^d \beta_i e_i, \beta_i \ge 0$, as an example to illustrate the main idea. Firstly, motivated by the Lie product formula or the splitting method, we can approximate $\phi_{b}^1$ by 
\begin{align}
    \phi_{b}^1 \approx 
    \big(\phi_{ e_1}^{\beta_1/n} \comp \phi_{ e_2}^{\beta_2 /n} \comp ... \comp \phi_{ e_d}^{\beta_d /n} \big)^{\comp n}, \quad n \in \mathbb{Z}_+,
\end{align}
with $n$ large enough. Secondly, each $\phi_{e_i}^{\beta_i /n}$ can be approximated by
\begin{align}
    \phi_{e_i}^{\beta_i /n} 
    \approx
    (\phi_{e_i}^{1})^{\comp p_i} \comp (\phi_{- e_i}^{\sqrt{2}})^{\comp q_i}
    \in \mathcal{H}_V
    ,\quad 
    p_i,q_i \in \mathbb{Z}_+
\end{align}
where $p_i$ and $q_i$ are non-negative integers such that $|p_i - q_i \sqrt{2} - \beta_i/n|$ is small enough according to the Kronecker's theorem \citep{Apostol1990Kronecker} as $\sqrt{2}$ is an irrational number. Finally, $\phi_{b}^1$ can be approximated by composing a sequence of flow maps in $V$. The case for $\phi^\tau_{Ax+b}$ and $\phi^\tau_{{\Sigma_{\boldsymbol{\alpha},\boldsymbol{\beta}}(x)}}$ in $H_1$ and $H_2$ can be done in the same spirit.

Then for Part 2, we note that the $g_L$ we constructed in Eq.~\eqref{eq:g_L} is similar to a feedforward neural network $g_L$ with width $d$ and depth $L$. The form of $g_L$ is motivated by a recent work of \citet{Duan2022Vanilla} which proved that vanilla feedforward leaky-ReLU networks with width $d$ can be a discretization of dynamic systems in dimension $d$. However, affine transformations in general networks are not necessarily OP diffeomorphisms, and one novelty of this paper is improving the technique to construct $P_i$ as flow maps. Importantly, making them flow maps helps with employing the construction in Part 1.

Remark that our theorems focused on the theory of function composition. If the functions are limited to linear functions, the composition is equivalent to matrix multiplication, and the corresponding UAP results remain. See Appendix~\ref{sec:V_linear} for the detailed statement for this linear case. The proof is going to be easy to follow, which only requires basic knowledge of linear algebra and Kronecker's approximation theorem in elementary number theory. We hope it eases the reader's burden of understanding our theorems and proofs.

\section{Proof of the construction Part 1}
\label{sec:warmup}

To warm up, we show some flow maps of autonomous ODEs below, with initial value $x(0)=x_0$,
\begin{align*}
    \dot x(t) = b  &\Rightarrow x(t) = \phi^t_{b}(x_0) =x_0 + bt,\\
    \dot x(t) = Ax(t) &\Rightarrow  x(t) =\phi^t_{Ax}(x_0) = e^{At} x_0,\\
    \dot x(t) = a\sigma_{0}(x(t)) &\Rightarrow  x(t) = \phi^t_{a\sigma_0(x)}(x_0) = e^{a t}\sigma_{e^{-a t}}(x_0),\\
    \dot x(t) = a \sigma_{{0}}(-x(t)) &\Rightarrow x(t) = \phi^t_{a \sigma_0(-x)}(x_0) = \sigma_{e^{-a t}}(x_0).
\end{align*}
Here $\sigma_0$ and $\sigma_{e^{-a t}}$ are ReLU and leaky-ReLU functions, respectively. Next, we provide some properties to verify a given map to be an affine flow map in $H_1$ or a leaky-ReLU flow map in $H_2$.

\subsection{Affine flows and leaky-ReLU flows} 

Consider the affine transformation $P: x \to Wx+b$ and examine conditions of $P$ to be a flow map. Generally, if $W$ is nonsingular and has real matrix logarithm $\ln(W)$, then $P$ is an affine flow map, as we can represent $P$ as $P(x) = Wx + b = \phi^1_{A x + \tilde b}$ where $A = \ln(W)$ and 
$\tilde b = \int_0^1 e^{A(\tau-1)} b d\tau$.
As it is hard to verify $\ln(W)$ is a real matrix \citep{Culver1966existence}, we are happy to construct some special matrix $W$. The following properties are useful.
\begin{proposition}\label{th:prop_linear}
    (1) Let $Q$ be a nonsingular matrix. If $x\to W x$ is an affine flow map then the map $ x\to Q W Q^{-1} x$, $ x\to W^T x$ and $x \to W^{-1}x$ also are.
    (2) Let $U$ be the upper triangular matrix below with $\lambda>0,$ then the map $x\to Ux$ is an affine flow map for arbitrary vector $w_{2:d}$,
    \begin{align}
        U = \left(
        \begin{matrix} 
            \lambda & w_{2:d}\\ 
            0 & I_{d-1}
        \end{matrix}
        \right). \quad
    \end{align}
\end{proposition}
Here $I_{d-1}$ is the $(d-1)$th order identity matrix. The property (1) is because $\ln(Q W Q^{-1}) = Q \ln(W) Q^{-1}$ and $\ln(W^T) = \ln(W)^T$. The property (2) can be obtained by employing the formula,
\begin{align}
    \ln
        \left(
    \begin{matrix} 
        \lambda & w_{2:d}\\ 
        0 & I_{d-1}
    \end{matrix}
    \right)
    =
    \left(
    \begin{matrix} 
        \ln(\lambda) & \frac{\ln(\lambda)}{\lambda-1} w_{2:d}\\ 
        0 & 0
    \end{matrix}
    \right), \quad \lambda \neq 1.
\end{align}
When $\lambda=1$, the formula is simplified as $\ln(U) = U - I_d$. 

Next, we consider the leaky-ReLU flow maps. 

By directly calculate the flow map  $\phi^\tau_{\Sigma_{\boldsymbol{\alpha},\boldsymbol{\beta}}(x)}$ with  $\boldsymbol{ \alpha},\boldsymbol{ \beta} \in \mathbb{R}^d$, we have
\begin{align}
    \phi^\tau_{\Sigma_{\boldsymbol{\alpha},\boldsymbol{\beta}}(x)} (x) = 
    \Sigma_{\boldsymbol{\tilde \alpha},\boldsymbol{\tilde \beta}} (x),
\end{align}
where $\boldsymbol{\tilde \alpha} =  (e^{\tau\alpha_1}, ..., e^{\tau\alpha_d})$ and $\boldsymbol{\tilde \beta} =  (e^{\tau\beta_1}, ..., e^{\tau\beta_d})$. The following property is implied.
\begin{proposition}\label{th:prop_leaky_relu}
    If $\boldsymbol{\tilde \alpha},\boldsymbol{\tilde \beta} \in (0,\infty)^d$, then the map $\Sigma_{\boldsymbol{\tilde \alpha},\boldsymbol{\tilde \beta}}$ is a leaky-ReLU flow map. 
\end{proposition}

\subsection{Application of Lie product formula}

\begin{theorem}[Lie product formula] \label{th:Lie}
    For all matrix $A,B \in \mathbb{R}^{d\times d}$, we have 
    \begin{align*}
        \text{e}^{A+B} = \lim_{n\to\infty} \Big( e^{A/n} e^{B/n} \Big)^n 
        = \lim_{n\to\infty} \Big( \phi^{1/n}_{Ax} \comp \phi^{1/n}_{Bx} \Big)^{\comp n}
    \end{align*}
\end{theorem}
Here $e^A$ denotes the matrix exponential of $A$, which is also the flow map $\phi^1_{Ax}$ of the autonomous system $x'(t)=A x(t)$. The proof can be found in \citet{Hall2015Matrix} for example and the formula can be extended to multi-component cases. The formula can also be derived from the operator splitting approach \citep{holden2010splitting}, which allows us to obtain the following result.
\begin{lemma}\label{th:comp_split}
    Let $v_i: \mathbb{R}^d \to \mathbb{R}^d, i=1,2,...,m$ be Lipschitz continuous funcitons, $v = \sum_{i=1}^m v_i$, $\Omega$ be a compact domain. For any $t>0$ and $\varepsilon>0$, there is a positive integer $n$, such that the flow map $\phi^t_v$ can be approximated by composition of flow maps $\phi^{t/n}_{v_i}$  , \emph{i.e.}
    \begin{align*}
        \| \phi_{v}^{t}(x)
        -
        \big(\phi_{v_1}^{t/n} \comp \phi_{v_2}^{t /n} \comp ... \comp \phi_{v_m}^{t /n} \big)^{\comp n}
        (x)\|
        < \varepsilon, \quad \forall x \in \Omega.
    \end{align*}
\end{lemma}

\subsection{Application of Kronecker's theorem}

\begin{theorem}[Kronecker's approximation theorem \citep{Apostol1990Kronecker}] \label{th:Kronecker}
    Let $\gamma \in \mathbb{R}$ be an irrational number, then for any $t\in \mathbb{R}$ and $\varepsilon>0$, there exist two integers $p$ and $q$ with $q>0$, such that $|\gamma q + p -t| < \varepsilon$. 
\end{theorem}

Although Kronecker's Theorem \ref{th:Kronecker} is proposed for approximating real numbers, we can employ it in the scenario of approximating the flow map $\phi_v^t$ as it contains a real time parameter $t$. Choosing $\gamma = -\sqrt{2}$, approximating $t$ by $p-q\sqrt{2}$, then we can approximate $\phi_v^t$ by $\phi_v^{p-q\sqrt{2}}$. Considering positive $t$, we have $p$ is positive as $q$ is. Then the property of flow maps,
\begin{align*}
    \phi_v^{p-q\sqrt{2}} = \phi_v^{p} \comp \phi_v^{-q\sqrt{2}} = \phi_v^{p} \comp \phi_{-v}^{q\sqrt{2}}
    =(\phi_{v}^{1})^{\comp p} \comp (\phi_{-v}^{\sqrt{2}})^{\comp q},
\end{align*}
allow us to prove the following result.
\begin{lemma}\label{th:comp_Kronecker}
    Let $v: \mathbb{R}^d \to \mathbb{R}^d$ be a Lipschitz continuous function, $\Omega$ be a compact domain. For any $t>0$ and $\varepsilon>0$, there exist two positive integers $p$ and $q$, such that the flow map $\phi^t_v$ can be approximated by $(\phi_{v}^{1})^{\comp p} \comp (\phi_{-v}^{\sqrt{2}})^{\comp q}$, \emph{i.e.}
    \begin{align}
        \| \phi_{v}^{t}(x)
        -
        (\phi_{v}^{1})^{\comp p} \comp (\phi_{-v}^{\sqrt{2}})^{\comp q}(x)\|
        < \varepsilon, \quad \forall x \in \Omega.
    \end{align}
\end{lemma}

\begin{corollary}\label{th:H_to_V}
    For any flow maps $h$ in $H_1 \cup H_2$, $\varepsilon>0$ and compact domain $\Omega\subset \mathbb{R}^d$, there is a sequence $\phi_1, \phi_2, ..., \phi_m$ in $V$ (Eq.~\ref{eq:V}) such that 
    \begin{align}
        \| h(x)
        -
        (\phi_{1} \comp \phi_2 \dots \comp \phi_m)(x)\|
        < \varepsilon, \quad \forall x \in \Omega.
    \end{align}
\end{corollary}

The result is obtained by directly employing Lemma~\ref{th:comp_split} and Lemma~\ref{th:comp_Kronecker} with the following splittings,
\begin{align}
    Ax + b &= 
    \sum_{i=1}^{d} 
    \sum_{j=1}^{d} 
    a_{ij} E_{ij} x + \sum_{i=1}^d b_i e_i,
    \\
    \Sigma_{\boldsymbol{ \alpha},\boldsymbol{ \beta}}(x)
    &=
    \sum_{i=1}^d \alpha_i \Sigma_{e_i,0}(x)
    + \sum_{i=1}^d \beta_i \Sigma_{0,e_i}(x).
\end{align}

\section{Proof of the construction Part 2}
\label{sec:construction}

This section provides the construction that OP diffeomorphisms can be approximated by composing a sequence of flow maps in $H_1 \cup H_2$. The construction contains three steps: (1) approximate OP diffeomorphisms by deep compositions using the splitting approach, (2) approximate each splitting component by composing flow maps in $H_1 \cup H_2$, (3) combine results to finish the construction.

\subsection{Approximate the OP diffeomorphism by deep compositions}

Employing results of \citet{Agrachev2010Dynamics} and \citet{Caponigro2011Orientation}, any OP diffeomorphism $\Psi$ can be approximated by flow maps of ODEs. Particularly, we can choose the ODEs as neural ODEs of the form 
\begin{align} \label{eq:NODE_v}
    \dot x(t) = 
    v(x(t),t) = 
    \sum\limits_{i=1}^N s_i(t)\sigma(w_i(t) \cdot x(t)+b_i(t)),
\end{align}
where the field function $v$ is a neural network with $N$ hidden neurons, the activation is chosen as the leaky-ReLU function $\sigma=\sigma_\alpha$ for some $\alpha \in (0,1)$, $s_i \in \mathbb{R}^d$, $w_i\in \mathbb{R}^d$ and $b_i\in \mathbb{R}$ are piecewise smooth functions of $t$. The universal approximation property of neural networks \citep{Cybenko1989Approximation} implies that $\Psi$ can be approximated by the flow map $\phi^\tau_{v}$ of Eq.~\eqref{eq:NODE_v} for some $\tau>0$ and $N \in \mathbb{Z}_+$ big enough.

Following the approach of \citet{Duan2022Vanilla}, we employ a proper splitting numerical scheme to discretize the neural ODE (\ref{eq:NODE_v}). Split the field $v$ as a summation of $Nd$ functions, $ v(x,t) = \sum_{i=1}^N \sum_{j=1}^d v_{ij}(x,t) {e}_j$, where ${e}_j$ is the $j$-th axis unit vector and
$ v_{ij}(x,t) = s_{ij}(t) \sigma(w_i(t) \cdot x+b_i(t))$
are scalar functions.
Then the numerical analysis theory of splitting methods \citep{holden2010splitting} ensures that the following composition $\Phi$ can approximate ${\phi}^\tau$ provided the time step $\Delta t:=\tau/n$ is sufficiently small,
\begin{align*}
    \Phi = &T_1 \comp T_2 \comp \cdots \comp T_n \\
    \equiv &
    (T_1^{(1,1)} \comp T_1^{(1,2)} \comp \dots \comp T_1^{(1,d)} \comp T_1^{(2,1)} \comp \dots \comp T_1^{(N,d)})
    \comp \\ &
    (T_2^{(1,1)} \comp T_2^{(1,2)} \comp \dots \comp T_2^{(1,d)} \comp T_2^{(2,1)} \comp \dots \comp T_2^{(N,d)})
    \comp \\&\dots \\ &\comp 
    (T_n^{(1,1)} \comp T_n^{(1,2)} \comp \dots \comp T_n^{(1,d)} \comp T_n^{(2,1)} \comp \dots \comp T_n^{(N,d)}),
\end{align*}
where the map $T_k^{(i,j)}: x \to y$ in each split step is
\begin{align}\label{eq:map_T_ijk}
    \left\{
    \begin{aligned} 
    & y^{(l)} = x^{(l)} , l \neq j,  \\
    & y^{(j)} = x^{(j)} + \Delta t v_{ij}(x,k \Delta t).
    \end{aligned}
    \right.
\end{align}
Here, the superscript in $x^{(l)}$ indicates the $l$-th coordinate of $x$. The map $T^{{i,j}}_k$ is given by the forward Euler discretization of $x'(t) = v_{i,j}(x(t),t) e_j$ in the interval $(k\Delta t, (k+1) \Delta t)$. Note that $v_{ij}$ is Lipschitz continuous on $\mathbb{R}^d$, hence the map  $T^{{i,j}}_k$ also is.

Below is the formal statement of the approximation in this step.

\begin{theorem}\label{th:OP_to_comp_Tijk}
Let $\Psi \in \text{Diff}_0(\Omega)$ be an orientation preserving diffeomorphism, $\Omega$ be a compact domain $\Omega \subset \mathbb{R}^d$. Then, for any $\varepsilon>0$, there is a sequence of transformations, $T_k^{(i,j)}$, is of the form Eq.~\eqref{eq:map_T_ijk} such that 
\begin{align}
    \|\Psi(x)-(T_1^{(1,1)} \comp T_1^{(1,2)} \comp 
    \dots \comp  T_n^{(N,d)})(x)\| \le \varepsilon, \quad
    \forall x \in \Omega.
    \nonumber
\end{align}
\end{theorem}

\subsection{Approximate each composition component by flow maps in $H_1$ and $H_2$}

Now we examine the map $T_k^{(i,j)}$ in each splitting step. Since all $T_k^{(i,j)}$ have the same structure (over a permutation), we only need to consider the case of $T_k^{(N,d)}$, which we simply denote as $T:x\to y$ of the form
\begin{align}\label{eq:map_T}
    T: \left\{
    \begin{aligned} 
    & y^{(i)} = x^{(i)} ,  i = 1,\cdots, d-1,\\
    & y^{(d)} = x^{(d)} + a  \sigma(w_{1} x^{(1)}+ \cdots +w_{d} x^{(d)} + b).
        \end{aligned}
    \right.
\end{align}
where $\sigma=\sigma_\alpha, \alpha \in (0,1)$, is the leaky-ReLU funciton, $a, b, w_1, ..., w_d \in \mathbb{R}$ are parameters. Since the time step $\Delta t$ in $T_k^{(i,j)}$ are small, we can assume the parameters satisfing $\max(1/\alpha,\alpha)|a w_d| < 1$.

\begin{lemma}\label{th:T_to_H}
Let $\alpha>0$ and $\max(1/\alpha,\alpha)|a w_d| < 1$, then the map $T$ in Eq.~\eqref{eq:map_T} is a composition of at most six flow maps in $H_1 \cup H_2$.
\end{lemma}

Noting that the case of $w_1=...=w_{d-1}=0$ is trivial, we can assume $w_1 \neq 0$ without loss of generality. Then, the bias parameter $b$ can be absorbed in $x^{(1)}$ using an affine flow map; hence we only need to consider the case of $b=0$. In addition, using the property of leaky-ReLU, $\sigma_\alpha(x) = - \alpha \sigma_{1/\alpha}(-x)$, we can further assume $w_1>0$. As a result, the map $T$ can be represented by the following composition,
\begin{align}
    T(x) = F_0 \comp F_1 \comp \cdots \comp F_5(x),
\end{align}
where each composition step is as follows,
\begin{align*} 
    \left(
        \begin{matrix} x^{(1)} \\ x^{(2:d-1)} \\ x^{(d)} 
        \end{matrix}\right)
    &\underrightarrow{F_0}
    \left(
        \begin{matrix} \nu \\ x^{(2:d-1)}\\ x^{(d)} 
        \end{matrix}\right)   
    \underrightarrow{F_1}
    \left(
        \begin{matrix} \sigma{(\nu)} \\ x^{(2:d-1)}\\ x^{(d)}
        \end{matrix}\right)
    \\
    &\underrightarrow{F_2}
    \left(
        \begin{matrix} \sigma{(\nu)} \\ x^{(2:d-1)}\\ x^{(d)}+ a \sigma{(\nu)} 
        \end{matrix}\right)
    \underrightarrow{F_3}
    \left(
        \begin{matrix} \nu \\ x^{(2:d-1)}\\ x^{(d)} + a \sigma{(\nu)} 
        \end{matrix}\right)
    \\&\underrightarrow{F_4}
    \left(
        \begin{matrix} \nu+w_{d} a \sigma{(\nu)} \\ x^{(2:d-1)}\\ x^{(d)}+ a \sigma{(\nu)} 
        \end{matrix}\right)
    \underrightarrow{F_5}
    \left(
        \begin{matrix} x^{(1)} \\ x^{(2:d-1)}\\ x^{(d)}+ a \sigma{(\nu)} 
        \end{matrix}\right).
\end{align*}
Here, $\nu:=w_{1}x^{(1)}+\cdots+w_{d}x^{(d)}$ and $x^{(2:d-1)}$ represent the elements $x^{(2)},...,x^{(d-1)}$.

We clarify that each component $F_i,i=0,\cdots,5,$ are flow maps in $H_1\cup H_2$. In fact, $F_0,F_2,F_5=F^{-1}_0$ are affine mappings,
\begin{align}
    F_0(x) &= \left(
        \begin{matrix} 
            w_{1} &w_{2:d}\\ 
            0 & I_{d-1}
        \end{matrix}
        \right) x, \\
    F_2(x) &= \left(
        \begin{matrix} 
            I_{d-1} &0\\ 
            (a,0_{2:d-1}) & 1
        \end{matrix}
        \right) x, \\
    F_5(x) &= \left(
        \begin{matrix} 
            1/w_1 &-w_{2:d}/w_1\\ 
            0 & I_{d-1}
        \end{matrix}
        \right) x,
\end{align}
where $I_{d-1}$ is the identity matrix, $(a,0_{2:d-1})=(a,0,...,0)$ with $d-2$ zeros. According to Proposition \ref{th:prop_linear}, they are flow maps in $H_1$. 
In addition, according to Proposition \ref{th:prop_leaky_relu}, $F_1, F_3$ and $F_4$ are leaky-ReLU flow maps in $H_2$ as 
\begin{align}
    F_1 &= \Sigma_{(\alpha,1_{2:d}),1_{1:d}}, \\
    F_3 &= \Sigma_{(1/\alpha,1_{2:d}),1_{1:d}}, \\
    F_4 &= \Sigma_{(1+w_d a \alpha,1_{2:d}), (1+w_d a,1_{2:d}) }.
\end{align}
Here, the condition $\max(1/\alpha,\alpha)|a w_d| < 1$ is used to ensure $1+w_d a \alpha >0$ and $1+w_d a>0$.

\subsection{Finish the construction}

Combining Theorem \ref{th:OP_to_comp_Tijk} and Lemma \ref{th:T_to_H} above, and using the fact in Lemma
~\ref{th:composition_approximation}, we have the following result.

\begin{theorem}\label{th:OP_to_comp_H}
Let $\Psi \in \text{Diff}_0(\Omega)$ be an orientation preserving diffeomorphism, $\Omega$ be a compact domain $\Omega \subset \mathbb{R}^d$. Then, for any $\varepsilon>0$, there is a sequence of flow maps, $h_1, h_2, ...,h_m, m \in \mathbb{Z}_+$, in $H = H_1 \cup H_2$ such that
\begin{align}
    \|\Psi(x)-(h_1\comp h_2\comp ...\comp h_m)(x)\| \le \varepsilon, \quad
    \forall x \in \Omega.
\end{align}
\end{theorem}

Then we can finish the construction for Theorem~\ref{th:main_V} by combining Corollary~\ref{th:H_to_V} and Theorem~\ref{th:OP_to_comp_H}.

\section{Compositional model for regular languages}
\label{sec:flow_grammar}

Regular languages constitute a basic type of language characterized by a symbolic formalism. Following the definition of matrix grammars and compositional matrix-space models (CMSM) of language in \cite{Rudolph2010Compositional}, here we define the flow grammars which serve as an alternative compositional model, we call it the compositional flow-space model (CFSM), for regular languages. 

It is well-known that deterministic finite automatons recognize exactly the set of regular languages (see Chapter 3.2 of \cite{Hopcroft2006Introduction} for example).
A deterministic finite automaton (DFA) $M$ is a 5-tuple, $M = (Q, \Sigma, \delta, q_0, F)$, consisting of a finite set of states $Q = \{q_0, . . . , q_{n-1}\}$, a finite set of input symbols called the alphabet $\Sigma$, a transition function $\delta: Q \times \Sigma\to Q$, an initial or start state $q_0$ and a set of accept states $F\subseteq Q$.
Let $w = s_1 s_2\cdots s_m$ be a string over the alphabet $\Sigma$. The automaton $M$ accepts the string $w$ if a sequence of states, $r_0, r_1, \cdots r_n$, exists in $Q$ such that $r_0 = q_0$, $r_{i+1} = \delta(r_i, s_{i+1})$ for $i = 0, \cdots m - 1$ and $r_{m}\in F$. Otherwise, it is said that the automaton rejects the string. The set of strings that $M$ accepts is the language recognized by $M$ and this language is denoted by $L(M)$.

\begin{definition}
    Let $\Omega=[0,1]^d, d\ge2,$ $p\in[1,\infty)$, and $\Sigma$ be an alphabet. A flow grammar $\mathcal{M}$ is defined as the pair $(\left\langle \cdot \right\rangle, A)$ where $\left\langle \cdot \right\rangle$ is a mapping from $\Sigma$ to $C(\Omega,\Omega)$ and $A = (g(\cdot),\rho(\cdot),\epsilon) \in C(\Omega,\mathbb{R}^d) \times C(\Omega,\mathbb{R}^+) \times \mathbb{R}^+$ characterizes the acceptance condition. The language generated by $\mathcal{M}$, denoted by $L(\mathcal{M})$, contains a string $w=s_1 s_2...s_m \in \Sigma^*$ exactly if
    \begin{align}
        I(w) = 
        \int_{\Omega} \rho(x)
        \big\|
            g(x) - 
        \left\langle w \right\rangle (x)
        \big\|^p  dx
        < \epsilon,
    \end{align}
    where $\left\langle w \right\rangle$ is defined as the composition 
    $\left\langle s_1 \right\rangle
    \comp\cdots
    \comp
    \left\langle s_m \right\rangle$.
    We will call a language $L$ flowable if $L = L(\mathcal{M})$ for some flow grammar $\mathcal{M}$.
\end{definition}
Remark that here we don't enforce the embedding $\left\langle s_i \right\rangle$ as a flow map in the definition because it can be approximated by flow maps according to the result of \cite{Brenier2003Approximation}. In addition, the density function $\rho(\cdot)$ can be generalized to more general functions. For example, choose $\rho(\cdot)$ as a Dirac measure at a point $x^*$, then the acceptance condition becomes $I(w)=\|g(x^*) - \left\langle w \right\rangle (x^*)
\big\|^p < \epsilon$ which avoids calculating integral values. Furthermore, similar to the matrix grammar defined in \cite{Rudolph2010Compositional}, one can also consider the acceptance condition such as the inner products $g_j(x^*_j)\cdot \left\langle w \right\rangle (x^*_j) \ge \gamma_j$ for some given functions $g_j \in C(\Omega,\mathbb{R}^d)$, points $x_j^* \in \mathbb{R}^d$ and numbers $\gamma_j \in \mathbb{R}$.

The following theorem indicates that regular languages are flowable.

\begin{theorem}
    \label{th:L_flowable}
    Let $\Omega=[0,1]^d, d\ge2,$ $p\in[1,\infty)$ be fixed. For any DFA $M$, there is a flow grammar $\mathcal{M}$ such that $L(M) = L(\mathcal{M})$.
\end{theorem}

The proof is constructive and here we use an example to show the main idea. The language recognized by the DFA $M=(\{S_1, S_2\},\{0,1\}, \delta, S_1, \{S_1\})$ in Figure~\ref{fig:DFA}(a) is the regular language given by the regular expression (1*)(0(1*)0(1*))*, where * is the Kleene star, \emph{e.g.}, 1* denotes any number (possibly zero) of consecutive ones. To construct the flow grammar $\mathcal{M}$, we choose two disconnected small cubes, $\Omega_1$ and $\Omega_2$, in $\Omega$. Then choose $\left\langle 1 \right\rangle$ be the indentity map and let $\left\langle 0 \right\rangle$ be a flow map which moves $\Omega_i$ to $\Omega_j$, $i \neq j$. The image $\left\langle 0 \right\rangle (x)$ for $x$ in $\Omega \setminus  (\Omega_1 \cup \Omega_2)$ is defined via mapping extention. It is obvious that a string $w = s_1 s_2\cdots s_m$ belongs to $L(M)$ if and only if $\left\langle w \right\rangle = 
\left\langle s_1 \right\rangle
\comp\cdots
\comp
\left\langle s_m \right\rangle$
keeps $\Omega_1$ unchanged. As a consequence, we can design $g(\cdot)$ as the identity map, $\rho(\cdot)$ as a continuous approximation of the characteristic function $\chi_{\Omega_1}$ for $\Omega_1$ and $\epsilon$ as a positive number small enough. The language generated by this flow grammar $\mathcal{M}$ satisfies the requirement that $L(M) = L(\mathcal{M})$.

Note that both the matrix grammars in \cite{Rudolph2010Compositional} and our constructed flow grammars are compositional models cover the regular languages. CMSM embedded symbols as linear mappings while CFSM allows us to use nonlinear mappings. Therefore, CFSM can be regarded as a natural generalization of CMSM.
Here we address their differences as well. Since the mapping space is significantly larger than the matrix space, CFSM is expected to have more powerful expressivity than CMSM. Noted that every regular language has a minimum number of required states (precisely the size $n$ of state set $Q$, which is not always easy to obtain), which must be explicitly involved in the DFA and CFSM models. For example, in CMSM, the order of the embedded matrix must be greater than the number of states, otherwise, it might be impossible to recognize the regular language. On the contrary, in CFSM, the required dimension $d$ (greater than or equal to 2) is independent of the number of required states. This is a benefit because the continuous function space is rich enough to capture any finite number of states.

\begin{figure}[htp!]
    \center
    \includegraphics[width=8cm]{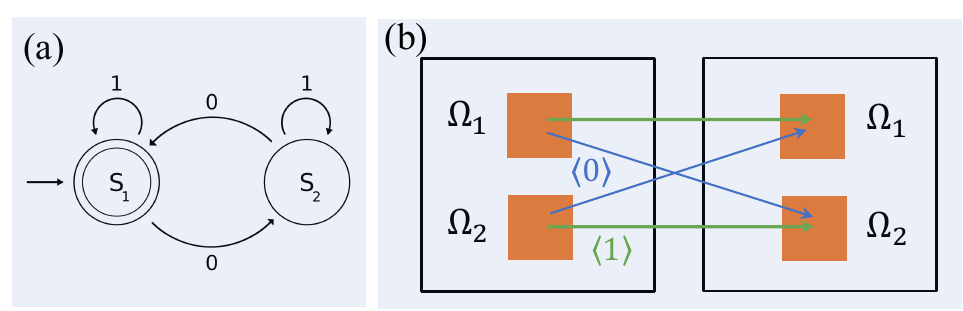}
    \caption{Example of (a) DFA and (b) flow grammar. }
    \label{fig:DFA}
\end{figure} 

\section{Conclusion}
\label{sec:conclusion}

This paper examined the approximation property of mapping composition from a sequential perspective. We proved, for the first time, that the universal approximation for diffeomorphisms and high-dimensional continuous functions can be achieved by using a finite number of sequential mappings. 
Our result implies that the universal approximations can be easily achieved. Importantly, the mappings used in our composition are flow maps of dynamical systems and do not increase the dimensions. However, our result is restricted to mappings on a compact domain. It is interesting to study whether it is possible to generalize this result to the case of mappings on unbounded domains.

Our Theorem~\ref{th:main-OP} was inspired by the fact of finite vocabulary in natural languages, where $V$ can be mimicked to a ``vocabulary'', $H_1$ and $H_2$ to ``phrases'', and $\mathcal{H}_V$ to ``sentences''. Our results provide a novel aspect for composite mappings, and we hope our findings could in turn inspire related research for the algorithm and modeling communities.
For example, one can embed words as nonlinear mappings instead of vectors or matrices in traditional models. We think there are at least three benefits of such embedding, \emph{i.e.}, the compositional flow-space model: (1) CFSM can capture the order of words. The compositional matrix-space model also has this property, but the classical vector embedding models do not. (2) CFSM has rich expression ability. The reason is that the continuous function space is infinite-dimensional, while the vector space and matrix space are finite-dimensional. (3) CFSM has the ability to cover regular grammars. Different with CFSM where the embedding dimension should be adaptive to the target languages, the dimension $d$ in our CFSM can be fixed as any integer $d\ge 2$ where the complexity of the languages is captured by the complexity of embedded mappings. 

Regular language is the simplest formal language in the Chomsky grammar system. For the CFSM considered in this paper, it is not difficult to imagine that it can represent many complex languages. However, it is very difficult to characterize the expressive range of a language model. In fact, even the matrix grammar can represent some special complex languages, and it is open whether the matrix grammar covers context-free languages \cite{Rudolph2010Compositional}. For this reason, this paper only discusses the regular languages and leaves the topic on more complex languages as future works. In addition, it's interesting to construct CFSM embedding in practice. Potential insights for building models for natural language processing are discussed in Appendix~\ref{sec:NLP}.

It should be noted that we use the terms of vocabulary, words, phrases, and sentences in this paper because they are the source of inspiration for proposing our main theorems in Section \ref{sec:main}. In addition, using these terms could help readers to understand our theorems and proofs more quickly. However, these terms can also be used in any compositional system, and the correspondence has a very loose connection to linguistics. It is interesting to explore whether are there any further similarities between mapping composition and linguistics.

It should be also noted that this paper focuses on the existence of a finite vocabulary and the constructed $V$ in Eq.~\eqref{eq:V} is not optimal. If a sequential composition of mappings in such $V$ is used to approximate functions in practical applications, the required sequence length may be extremely large. However, in practical applications, it is often only necessary to approximate a certain small set of continuous functions, hence designing an efficient vocabulary for them would be a fascinating future direction.






\section*{Acknowledgements}



This research is supported by the National Natural Science Foundation of China (Grant No. 12201053 and No. 12171043). I thank anonymous reviewers for their valuable comments and useful suggestions, and I am deeply grateful to my wife for allowing me to focus on research instead of struggling with life.

\section*{Impact Statement}




This paper presents work whose goal is to advance the field of Machine Learning. There are many potential societal consequences of our work, none of which we feel must be specifically highlighted here.


\bibliography{tex/refs.bib}
\bibliographystyle{icml2024}

\newpage
\appendix
\onecolumn




\section{Additional lemmas}

It is well known that an ODE system can be approximated by many numerical methods. Particularly, we use the splitting approach \citep{holden2010splitting}. Let $v(x,t)$ be the summation of several functions,
\begin{align}
    v(x,t) = \sum_{j=1}^J v_j(x,t), J \in \mathbb{Z}_+.
\end{align}
For a given time step $\Delta t$, we define the iteration as
\begin{align}\label{eq:iteration_T1}
    x_{k+1} = T_k x_k = T_k^{(J)} \circ \dots \circ T_k^{(2)} \circ T_k^{(1)} x_k,
\end{align}
where the map $T_k^{(j)}: x \to y$ is
\begin{align}
    y \equiv T_k^{(j)}(x) = x + \Delta t v_j(x,t_k)
    \quad
    t_k = k \Delta t.
\end{align}

\begin{lemma}\label{th:split_approach}
Let all $v_j(x,t), j=1,2,...,J,$ and $v(x,t), (x,t)\in \mathbb{R}^d \times [0,\tau],$ be piecewise constant (w.r.t. $t$) and $L$-Lipschitz ($L>0$). Then, for any $\tau>0$, $\varepsilon>0$ and $x_0$ in a compact domain $\Omega$, there exist a positive integer $n$ and $\Delta t = \tau/n < 1$ such that $\|x(\tau) - x_n\| \le \varepsilon$.
\end{lemma}
\begin{proof}
Without loss of generality, we only consider $J=2$. (The general $J$ case can be proven accordingly.) In addition, we assume $v_j$ are constant w.r.t. $t$ in each interval $[t_k,t_{k+1})$ the time step, \emph{i.e.} $v(x,t) = v(x,t_k), t\in [t_k,t_{k+1})$. This can be arrived at by choosing small enough $\Delta t$ and adjusting the time step to match the piecewise points of $v_j$. Thus, we have
\begin{align*}
    x_{k+1} &= T_k^{(2)} (x_k + \Delta t v_1(x_k,t_k))\\
            & = x_k + \Delta t v_1(x_k,t_k) + \Delta t v_2(x_k + \Delta t v_1(x_k, t_k),t_k)\\
            &= x_k + \Delta t (v_1(x_k,t_k) + v_2(x_k,t_k)) 
    + \Delta t^2 R_k.
\end{align*}
Since $v_i(x,t)$ are Lipschitz, the residual term $R_k$ is bounded by a constant $R$ that is independent of $k$. In fact, we have
\begin{align*}
    \|R_k\| &= \|v_2(x_k + \Delta t v_1(x_k, t_k),t_k) - v_2(x_k,t_k)\|/\Delta t \\
    &\le
    L \|v_1(x_k, t_k)\| 
    \le L( \|v_1(0, t_k)\| + L \|x_k\|).
\end{align*}
Let $V := \sup \{\|v_j\| | t \in (0,\tau)\}, X:= \sup \{\|x_0\| | x_0 \in \Omega\}$, then we have
\begin{align}
    \|x_{k+1}\| \le (1+L\Delta t) \|x_k\| + \Delta t^2 L(V + L \|x_k\|).
\end{align}
As a result, $\|x_k\|$ is bounded by $B:= (X + \frac{V \Delta t}{1+L}) e^{L(1+L)\tau}$ and $\|R_k\|$ is bounded by $R:=L(V+LB)$. 

Using the integral form of the ODE
and defining the error as $e_k:=x_k-x(t_k)$, we have the following estimation:
\begin{align*}
    \|e_{k+1}\|
    &=
    \|e_k + 
    \int_{t_k}^{t_{k+1}} (v(x_k,t_k) - v(x(t),t) ) dt
    + R_k \Delta t^2 \|\\
    &\le 
    \|e_k\| + 
    \int_{t_k}^{t_{k+1}} \|v(x(t),t_k) - v(x_k,t_k)\| dt
    + \|R_k\| \Delta t^2\\
    &\le (1+L\Delta t)\|e_k\| + R \Delta t^2.
\end{align*}
Employing the inequality $(1+L\Delta t)^k \le e^{Lk \Delta t } \le e^{L\tau}$ and the initial error $e_0=0$, we have
\begin{align}
    \|e_{k}\|
    \le
    (1+L\Delta t)^k\|e_0\|
    +
    \frac{R \Delta t^2}{L\Delta t}[(1+L\Delta t)^k-1]
    \le 
    R \Delta t (e^{L\tau}-1)/L .
\end{align}
For any $\varepsilon >0$, let $n \ge [\frac{R\tau e^{L\tau}}{L \varepsilon} ]$, then we have $\|x(t_k)-x_k\| \le \varepsilon$, which finishes the proof.
\end{proof}

\section{Proofs of lemmas and propositions}

\subsection{Proof of Lemma \ref{th:composition_approximation}}

\textit{\textbf{Lemma \ref{th:composition_approximation}.}
    Let map $T = F_1 \comp ... \comp F_n$ be a composition of $n$ continuous functions $F_i$ defined on an open domain $D_i$, and let $\mathcal{F}$ be a continuous function class that can uniformly approximate each $F_i$ on any compact domain $\mathcal{K}_i \subset D_i$. Then, for any compact domain $\mathcal{K} \subset D_1$ and $\varepsilon >0$, there are $n$ functions $\tilde F_1, ..., \tilde F_n$ in $\mathcal{F}$ such that
    \begin{align}
            \|T(x) - \tilde F_1 \comp ... \comp \tilde F_n (x)\|
            \le \varepsilon,
            \quad
            \forall x \in \mathcal{K}.
    \end{align}
}

\begin{proof}
It is enough to prove the case of $n=2$. (The case of $n>2$ can be proven by the method of induction, as $T$ can be expressed as the composition of two functions, $T = F_n \circ T_{n-1}$, with $T_{n-1} = F_{n-1} \circ ... \circ F_1$.) According to the definition, we have $F_1(D_1) \subset D_2$. Since $D_2$ is open and $F_1(\mathcal{K})$ is compact, we can choose a compact set $\mathcal{K}_2 \subset D_2$ such that $\mathcal{K}_2 \supset \{F_1(x) + \delta_0  y: x\in \mathcal{K}, \|y\|<1 \} $ for some $\delta_0>0$ that is sufficiently small.

According to the continuity of $F_2$, there is a $\delta \in (0,\delta_0)$ such that
\begin{align*}
        \|F_2(y) - F_2(y')\| &\le \varepsilon/2, \forall y,y' \in \mathcal{K}_2,
    \end{align*}
provided $\|y-y'\| \le \delta$.
The approximation property of $\mathcal{F}$ allows us to choose $\tilde F_1, \tilde F_2 \in \mathcal{F}$ such that
\begin{align*}
        \|\tilde F_1(x) - F_1(x)\| &\le \delta < \delta_0, \quad \forall x \in \mathcal{K}, \\
        \|\tilde F_2(y) - F_2(y)\| &\le \varepsilon/2, \quad \forall y \in \mathcal{K}_2.
    \end{align*}
As a consequence, for any $x \in \mathcal{K}$, we have $F_1(x), \tilde F_1(x) \in \mathcal{K}_2$ and
\begin{align*}
        \|F_2 \circ F_1(x) - \tilde F_2 \circ \tilde F_1(x)\|
        &\le
        \|F_2 \circ F_1(x) - F_2 \circ \tilde F_1(x)\|
        +
        \|F_2 \circ \tilde F_1(x) - \tilde F_2 \circ \tilde F_1(x)\|\\
        &\le
        \varepsilon/2 + \varepsilon/2 = \varepsilon.
    \end{align*}
\end{proof}

\subsection{Proof of Proposition \ref{th:prop_linear}}

\textit{\textbf{Proposition \ref{th:prop_linear}.}
(1) Let $Q$ be a nonsingular matrix. If $x\to W x$ is an affine flow map then the map $ x\to Q W Q^{-1} x$, $ x\to W^T x$ and $x \to W^{-1}x$ also are.
(2) Let $U$ be an upper triangular matrix below with $\lambda>0,$ then the map $x\to Ux$ is an affine flow map for arbitrary vector $w_{2:d}$,
\begin{align}
    U = \left(
    \begin{matrix} 
        \lambda & w_{2:d}\\ 
        0 & I_{d-1}
    \end{matrix}
    \right). \quad
\end{align}
}

\begin{proof}
    (1) It is because $\ln(Q W Q^{-1}) = Q \ln(W) Q^{-1}$, $\ln(W^T) = \ln(W)^T$ and $\ln(W^{-1}) = -\ln(W)$ are real as $\ln(W)$ is real. 
    (2) It can be obtained by employing the formula,
\begin{align}
    \ln
        \left(
    \begin{matrix} 
        \lambda & w_{2:d}\\ 
        0 & I_{d-1}
    \end{matrix}
    \right)
    =
    \left(
    \begin{matrix} 
        \ln(\lambda) & \frac{\ln(\lambda)}{\lambda-1} w_{2:d}\\ 
        0 & 0
    \end{matrix}
    \right), \quad \lambda \neq 1.
\end{align}
When $\lambda=1$, the formula is simplified as $\ln(U) = U - I_d$. 

\end{proof}

\subsection{Proof of Proposition \ref{th:prop_leaky_relu}}

\textit{\textbf{Proposition \ref{th:prop_leaky_relu}.}
    If $\boldsymbol{\tilde \alpha},\boldsymbol{\tilde \beta} \in (0,\infty)^d$, then the map $\Sigma_{\boldsymbol{\tilde \alpha},\boldsymbol{\tilde \beta}}$ is a leaky-ReLU flow map. 
}

\begin{proof}
    By directly calculate the flow map  $\phi^\tau_{\Sigma_{\boldsymbol{\alpha},\boldsymbol{\beta}}(x)}$ with  $\boldsymbol{ \alpha},\boldsymbol{ \beta} \in \mathbb{R}^d$, we have
\begin{align}
    \phi^\tau_{\Sigma_{\boldsymbol{\alpha},\boldsymbol{\beta}}(x)} (x) = 
    \Sigma_{\boldsymbol{\tilde \alpha},\boldsymbol{\tilde \beta}} (x),
\end{align}
where $\boldsymbol{\tilde \alpha} =  (e^{\tau\alpha_1}, ..., e^{\tau\alpha_d})$ and $\boldsymbol{\tilde \beta} =  (e^{\tau\beta_1}, ..., e^{\tau\beta_d})$. 
Choosing $\alpha_i = \ln(\tilde \alpha_i), \beta_i = \ln(\tilde \beta_i)$ and $\tau = 1$, we can finish the proof.
\end{proof}

\subsection{Proof of Lemma \ref{th:comp_split}}

\textit{\textbf{Theorem \ref{th:Lie}. (Lie product formula)}
    For all matrix $A,B \in \mathbb{R}^{d\times d}$, we have 
    \begin{align}
        \text{e}^{A+B} = \lim_{n\to\infty} \Big( e^{A/n} e^{B/n} \Big)^n 
        = \lim_{n\to\infty} \Big( \phi^{1/n}_{Ax} \comp \phi^{1/n}_{Bx} \Big)^{\comp n}
    \end{align}
}

\begin{proof}
    The proof can be found in \cite{Hall2015Matrix}.
\end{proof}

\textit{\textbf{Lemma \ref{th:comp_split}.}
    Let $v_i: \mathbb{R}^d \to \mathbb{R}^d, i=1,2,...,m$ be Lipschitz continuous funcitons, $v = \sum_{i=1}^m v_i$, $\Omega$ be a compact domain. For any $t>0$ and $\varepsilon>0$, there is a positive integers $n$, such that the flow map $\phi^t_v$ can be approximated by composition of flow maps $\phi^{t/n}_{v_i}$  , \emph{i.e.}
    \begin{align}
        \| \phi_{v}^{t}(x)
        -
        \big(\phi_{v_1}^{t/n} \comp \phi_{v_2}^{t /n} \comp ... \comp \phi_{v_m}^{t /n} \big)^{\comp n}
        (x)\|
        < \varepsilon, \quad \forall x \in \Omega.
    \end{align}
}

\begin{proof}
    It's a special case of Lemma~\ref{th:split_approach} with a velocity field $v_i$ independent on $t$.
\end{proof}

\subsection{Proof of Lemma \ref{th:comp_Kronecker}}

\textit{\textbf{Theorem \ref{th:Kronecker}. (Kronecker's approximation theorem)}
    Let $\gamma \in \mathbb{R}$ be an irrational number, then for any $t\in \mathbb{R}$ and $\varepsilon>0$, there exist two integers $p$ and $q$ with $q>0$, such that $|\gamma q + p -t| < \varepsilon$. 
}
\begin{proof}
    The proof can be found in \cite{Apostol1990Kronecker}.
\end{proof}

\textit{\textbf{Lemma \ref{th:comp_Kronecker}.}
    Let $v: \mathbb{R}^d \to \mathbb{R}^d$ be a Lipschitz continuous function, $\Omega$ be a compact domain. For any $t>0$ and $\varepsilon>0$, there exist two positive integers $p$ and $q$, such that the flow map $\phi^t_v$ can be approximated by $(\phi_{v}^{1})^{\comp p} \comp (\phi_{-v}^{\sqrt{2}})^{\comp q}$, \emph{i.e.}
    \begin{align}
        \| \phi_{v}^{t}(x)
        -
        (\phi_{v}^{1})^{\comp p} \comp (\phi_{-v}^{\sqrt{2}})^{\comp q}(x)\|
        < \varepsilon, \quad \forall x \in \Omega.
    \end{align}
}
\begin{proof}
    Since the field $v$ is Lipschitz and the domain $\Omega$ is compact, there exist a constant $C>0$ such that
    \begin{align}
        \|\phi_v^{t_2}(x_0) - \phi_v^{t_1}(x_0)\| \le \int_{t_1}^{t_2} \|v(x(t))\| dt < C |t_2 - t_1|, 
        \quad \forall x_0 \in \Omega.
    \end{align}

    Employing the Kronecker's Theorem \ref{th:Kronecker} with $\gamma = -\sqrt{2}$, approximating $t$ by $p-q\sqrt{2}$ such that
    \begin{align}
        |p-q\sqrt{2} - t| < \varepsilon/C,
    \end{align}
    then we have
    \begin{align}
        \|\phi_v^t(x) - \phi_v^{p-q\sqrt{2}}(x)\| < \varepsilon, \quad  \forall x \in \Omega.
    \end{align} 
    As $t$ is positive, we have $p$ is positive as $q$ is. The following representation of the flow maps finishes the proof,
\begin{align}
    \phi_v^{p-q\sqrt{2}} = \phi_v^{p} \comp \phi_v^{-q\sqrt{2}} = \phi_v^{p} \comp \phi_{-v}^{q\sqrt{2}}
    =(\phi_{v}^{1})^{\comp p} \comp (\phi_{-v}^{\sqrt{2}})^{\comp q}.
\end{align}

\end{proof}

\textit{\textbf{Corollary \ref{th:H_to_V}.}
    For any flow maps $h$ in $H_1 \cup H_2$, $\varepsilon>0$ and compact domain $\Omega\subset \mathbb{R}^d$, there is a sequence $\phi_1, \phi_2, ..., \phi_m$ in $V$ (Eq.~\ref{eq:V}) such that 
    \begin{align}
        \| h(x)
        -
        (\phi_{1} \comp \phi_2 \dots \comp \phi_m)(x)\|
        < \varepsilon, \quad \forall x \in \Omega.
    \end{align}
}

\begin{proof}
    The proof is finished by directly employing Lemma~\ref{th:comp_split} and Lemma~\ref{th:comp_Kronecker} with the following splittings,
\begin{align}
    Ax + b &= 
    \sum_{i=1}^{d} 
    \sum_{j=1}^{d} 
    a_{ij} E_{ij} x + \sum_{i=1}^d b_i e_i,
    \quad
    \Sigma_{\boldsymbol{ \alpha},\boldsymbol{ \beta}}(x)
    =
    \sum_{i=1}^d \alpha_i \Sigma_{e_i,0}(x)
    + \sum_{i=1}^d \beta_i \Sigma_{0,e_i}(x).
\end{align}
\end{proof}

\subsection{Proof of Lemma \ref{th:T_to_H} }

\textit{\textbf{Lemma \ref{th:T_to_H}.}
Let $\alpha>0$ and $\max(1/\alpha,\alpha)|a w_d| < 1$, then the map $T$ in Eq.~\eqref{eq:map_T} is a composition of at most six flow maps in $H_1 \cup H_2$.
}

\begin{proof}
    Recall the map $T:x\to y$ is of the form
\begin{align}
    T: \left\{
    \begin{aligned} 
    & y^{(i)} = x^{(i)} ,  i = 1,\cdots, d-1,\\
    & y^{(d)} = x^{(d)} + a  \sigma(w_{1} x^{(1)}+ \cdots +w_{d} x^{(d)} + b).
        \end{aligned}
    \right.
\end{align}
where $\sigma=\sigma_\alpha$ is the leaky-ReLU funciton, $a, b, w_1, ..., w_d \in \mathbb{R}$ are parameters. We construct the composition flow maps in three cases.

(1) The case of $w_1=...=w_{d}= 0$. In this case, $T$ is already an affine flow map in $H_1$.

(2) The case of $w_1=...=w_{d-1}= 0, w_d \neq 0$. In this case, we only need to consider the last coordinate as the first $d-1$ coordinates are kept. According to
\begin{align}
    y^{(d)} = x^{(d)} + a  \sigma_{\alpha}(w_{d} x^{(d)} + b)
    =
    (x^{(d)}+ \tfrac{b}{w_d}) + a  \sigma_{\alpha}(w_{d} (x^{(d)} + \tfrac{b}{w_d})) -  \tfrac{b}{w_d},
\end{align}
we can assume $b=0$ as it can be absorbed in an affine flow map. Let $\tilde \alpha = 1 + \alpha a w_d >0, \tilde \beta = 1 + a w_d >0$, as $\max(1/\alpha,\alpha)|a w_d| < 1$, we have the following representation,
\begin{align}
    x^{(d)} + a  \sigma_{\alpha}(w_{d} x^{(d)}) = 
    \begin{cases}
        \sigma_{\tilde \alpha, \tilde \beta} (x^{(d)}), & w_d <0,\\
        \sigma_{\tilde \beta, \tilde \alpha} (x^{(d)}), & w_d >0,
    \end{cases}
\end{align}
which is a leaky-ReLU flow map in $H_3$ either $w_d>0$ or $w_d<0$.

(3) The case of $w_i \neq 0$ for some $i=1,...,d-1$. We only show the case of $w_1 \neq 0$ without loss of generality. Same with (1), we can absorb $b$ in $x^{(1)}$ using an affine flow map; hence we only need to consider the case of $b=0$. In addition, using the property of leaky-ReLU, 
\begin{align}\label{eq:prop_leaky_ReLU}
    \sigma_\alpha(x) = - \alpha \sigma_{1/\alpha}(-x)
\end{align}
$\sigma_\alpha(x) = - \alpha \sigma_{1/\alpha}(-x)$, we can further assume $w_1>0$. (If $w_1<0$, we change $w$ to $-w$, $\alpha$ to $1/\alpha$, $a$ to $a \alpha$, which does not change the map $T$). As a result, the map $T$ can be represented by the following composition,
\begin{align}
    T(x) = F_0 \comp F_1 \comp \cdots \comp F_5(x),
\end{align}
where each composition step is as follows,
\begin{align*} 
    \left(
        \begin{matrix} x^{(1)} \\ x^{(2:d-1)} \\ x^{(d)} 
        \end{matrix}\right)
    &\underrightarrow{F_0}
    \left(
        \begin{matrix} \nu \\ x^{(2:d-1)}\\ x^{(d)} 
        \end{matrix}\right)   
    \underrightarrow{F_1}
    \left(
        \begin{matrix} \sigma{(\nu)} \\ x^{(2:d-1)}\\ x^{(d)}
        \end{matrix}\right)
    \underrightarrow{F_2}
    \left(
        \begin{matrix} \sigma{(\nu)} \\ x^{(2:d-1)}\\ x^{(d)}+ a \sigma{(\nu)} 
        \end{matrix}\right)
    \underrightarrow{F_3}
    \left(
        \begin{matrix} \nu \\ x^{(2:d-1)}\\ x^{(d)} + a \sigma{(\nu)} 
        \end{matrix}\right)
    \\&\underrightarrow{F_4}
    \left(
        \begin{matrix} \nu+w_{d} a \sigma{(\nu)} \\ x^{(2:d-1)}\\ x^{(d)}+ a \sigma{(\nu)} 
        \end{matrix}\right)
    \underrightarrow{F_5}
    \left(
        \begin{matrix} x^{(1)} \\ x^{(2:d-1)}\\ x^{(d)}+ a \sigma{(\nu)} 
        \end{matrix}\right).
\end{align*}
Here, $\nu:=w_{1}x^{(1)}+\cdots+w_{d}x^{(d)}$ and $x^{(2:d-1)}$ represent the elements $x^{(2)},...,x^{(d-1)}$.
We clarify that each component $F_i,i=0,\cdots,5,$ are flow maps in $H_1\cup H_2$. 

In fact, $F_0,F_2,F_5=F^{-1}_0$ are affine transformations,
\begin{align*}
    F_0(x) = \left(
        \begin{matrix} 
            w_{1} &w_{2:d}\\ 
            0 & I_{d-1}
        \end{matrix}
        \right) x, \quad
    F_2(x) = \left(
        \begin{matrix} 
            I_{d-1} &0\\ 
            (a,0_{2:d-1}) & 1
        \end{matrix}
        \right) x, \quad
    F_5(x) = \left(
        \begin{matrix} 
            1/w_1 &-w_{2:d}/w_1\\ 
            0 & I_{d-1}
        \end{matrix}
        \right) x,
\end{align*}
where $I_{d-1}$ is the identity matrix, $(a,0_{2:d-1})=(a,0,...,0)$ with $d-2$ zeros. According to Proposition \ref{th:prop_linear}, they are flow maps in $H_1$. 
In addition, $F_1, F_3$ and $F_4$ are leaky-ReLU flow maps in $H_2$ as 
\begin{align}
    F_1 = \Sigma_{(\alpha,1_{2:d}),1_{1:d}}, \quad
    F_3 = \Sigma_{(1/\alpha,1_{2:d}),1_{1:d}}, \quad
    F_4= \Sigma_{(1+w_d a \alpha,1_{2:d}), (1+w_d a,1_{2:d}) }.
\end{align}
Here, the condition $\max(1/\alpha,\alpha)|a w_d| < 1$ is used to ensure $1+w_d a \alpha >0$ and $1+w_d a>0$, no matter whether Eq.~\eqref{eq:prop_leaky_ReLU} is uesd.

\end{proof}

\section{Proof of the main theorems}

\subsection{Proof of Theorem \ref{th:OP_to_comp_Tijk}}

\textit{\textbf{Theorem \ref{th:OP_to_comp_Tijk}.}
Let $\Psi \in \text{Diff}_0(\Omega)$ be an orientation preserving diffeomorphism, $\Omega$ be a compact domain $\Omega \subset \mathbb{R}^d$. Then, for any $\varepsilon>0$, there is a sequence of transformations, $T_k^{(i,j)}$, is of the form Eq.~\eqref{eq:map_T_ijk} such that 
\begin{align}
    \|\Psi(x)-(T_1^{(1,1)} \comp T_1^{(1,2)} \comp \dots \comp T_1^{(N,d)}
    \comp \dots \comp
    T_n^{(1,1)} \comp T_n^{(1,2)} \comp \dots \comp  T_n^{(N,d)})(x)\| \le \varepsilon, \quad
    \forall x \in \Omega.
    \nonumber
\end{align}
}

\begin{proof}

(1) Firstly, employed results of \citet{Agrachev2010Dynamics} and \citet{Caponigro2011Orientation}, any OP diffeomorphism $\Psi$ can be approximated by flow map of ODEs. Particularly, we can choose the ODEs as neural ODEs are of the form 
\begin{align} 
    x'(t) = 
    v(x(t),t) = 
    \sum\limits_{i=1}^N s_i(t)\sigma(w_i(t) \cdot x(t)+b_i(t)),
\end{align}
where the field function $v$ is a neural network with $N$ hidden neurons, the activation is chosen as the leaky-ReLU function $\sigma=\sigma_\alpha$ for some $\alpha \in (0,1)$, $s_i \in \mathbb{R}^d$, $w_i\in \mathbb{R}^d$ and $b_i\in \mathbb{R}$ are piecewise constant functions of $t$. The universal approximation property of neural networks \cite{Cybenko1989Approximation} implies that, for any $\varepsilon>0$, there exist $s_i \in \mathbb{R}^d$, $w_i\in \mathbb{R}^d, b_i \in \mathbb{R}, \tau >0$ and $N \in \mathbb{Z}_+$, such that 
\begin{align}
   \|\Psi(x) - \phi_v^\tau(x)\| < \varepsilon/2, \quad \forall x \in \Omega,
\end{align}
where $\phi^\tau_{v}$ is the flow map of Eq.~\eqref{eq:NODE_v}.

(2) Following the approach of \cite{Duan2022Vanilla}, we employ a proper splitting numerical scheme to discretize the neural ODE (\ref{eq:NODE_v}). Split the field $v$ as a summation of $Nd$ functions, $ v(x,t) = \sum_{i=1}^N \sum_{j=1}^d v_{ij}(x,t) {e}_j$, where ${e}_j$ is the $j$-th axis unit vector and
$ v_{ij}(x,t) = s_{ij}(t) \sigma(w_i(t) \cdot x+b_i(t))$
are scalar Lipschitz functions.
Then Lemma~\ref{th:split_approach} implies that there is a $n\in\mathbb{Z}_+$ big enough such that 
\begin{align}
    \|\phi_v^\tau(x) - \Phi(x)\| < \varepsilon/2, \quad \forall x \in \Omega,
 \end{align}
where
\begin{align*}
    \Phi &= T_1 \comp T_2 \comp \cdots \comp T_n \\
    &\equiv
    (T_1^{(1,1)} \comp T_1^{(1,2)} \comp \dots \comp T_1^{(N,d)})
    \comp
    (T_2^{(1,1)} \comp T_2^{(1,2)} \comp \dots \comp T_2^{(N,d)})
    \comp \dots \comp
    (T_n^{(1,1)} \comp T_n^{(1,2)} \comp \dots  \comp T_n^{(N,d)}),
\end{align*}
and the map $T_k^{(i,j)}: x \to y$ is of the form
\begin{align}
    \left\{
    \begin{aligned} 
    & y^{(l)} = x^{(l)} , l \neq j,  \\
    & y^{(j)} = x^{(j)} + \Delta t v_{ij}(x,k \Delta t).
    \end{aligned}
    \right.
\end{align}
Here, the superscript in $x^{(l)}$ indicates the $l$-th coordinate of $x$. 

(3) Combining the above two parts, we finish the proof.
\end{proof}

\subsection{Proof of Theorem \ref{th:OP_to_comp_H}}

\textit{\textbf{Theorem \ref{th:OP_to_comp_H}.}
Let $\Psi \in \text{Diff}_0(\Omega)$ be an orientation preserving diffeomorphism, $\Omega$ be a compact domain $\Omega \subset \mathbb{R}^d$. Then, for any $\varepsilon>0$, there is a sequence of flow maps, $h_1, h_2, ...,h_m, m \in \mathbb{Z}_+$, in $H = H_1 \cup H_2$ such that
\begin{align}
    \|\Psi(x)-(h_1\comp h_2\comp ...\comp h_m)(x)\| \le \varepsilon, \quad
    \forall x \in \Omega.
\end{align}
}

\begin{proof}
According to Theorem \ref{th:OP_to_comp_Tijk}, there is a sequence of transformations, $T_k^{(i,j)}$, is of the form Eq.~\eqref{eq:map_T_ijk} such that 
\begin{align}
    \|\Psi(x)-(T_1^{(1,1)} \comp T_1^{(1,2)} \comp \dots \comp T_1^{(N,d)}
    \comp \dots \comp
    T_n^{(1,1)} \comp T_n^{(1,2)} \comp \dots \comp  T_n^{(N,d)})(x)\| \le \varepsilon, \quad
    \forall x \in \Omega.
    \nonumber
\end{align}
Here $n$ can be choosed large enough such that $\max(1/\alpha,\alpha) C^2 \Delta t < 1, \Delta t = \tau/n$, where 
\begin{align}
    C = \max_{t \in [0,\tau]} \big\{|s_{ij}(t)|, |w_{ij}(t)| ~|~ i,j=1,2,...,d \big\}.
\end{align}
Since $s_i, w_i$ are piecewise constant functions, the constant $C$ is finite. Then according to Lemma \ref{th:T_to_H}, each $T_k^{(i,j)}$ is a composition of at most six flow maps in $H_1 \cup H_2$. As a consequence, we finish the proof by relabelling the index of the used flow maps.

\end{proof}

\subsection{Proof of Theorem \ref{th:main_V}}

\textit{\textbf{Theorem \ref{th:main_V}.}
    Let $\Psi \in \text{Diff}_0(\Omega)$ be an orientation preserving diffeomorphism, $\Omega$ be a compact domain $\Omega \subset \mathbb{R}^d$. Then, for any $\varepsilon>0$, there is a sequence of flow maps, $\phi_1, \phi_2, ...,\phi_n \in V, n \in \mathbb{Z}_+$, such that
    \begin{align}
        \|\Psi(x)-(\phi_1\comp \phi_2\comp ...\comp\phi_n)(x)\| \le \varepsilon, \quad
        \forall x \in \Omega.
    \end{align}
}

\begin{proof}

(1) According to Theorem \ref{th:OP_to_comp_H}, there is a sequence of flow maps, $h_1, h_2, ...,h_m, m \in \mathbb{Z}_+$, in $H = H_1 \cup H_2$ such that
\begin{align}
    \|\Psi(x)-(h_1\comp h_2\comp ...\comp h_m)(x)\| \le \varepsilon/2, \quad
    \forall x \in \Omega.
\end{align}

(2) According to Corollary~\ref{th:H_to_V}, each $h_i$ can be universal approximation by $\mathcal{H}_V$, \emph{i.e.}, for any $\varepsilon_i>0$ and compact domain $\Omega_i$, there is a sequence of flow maps, $\phi_{i,1},..., \phi_{i,n_i} \in V$, such that  
\begin{align}
    \|h_i(x)- (\phi_{i,1} \comp \phi_{i,2} \comp ...\comp \phi_{i,n_i})(x)\| \le \varepsilon_i, \quad
    \forall x \in \Omega_i.
\end{align}

(3) According to Lemma~\ref{th:composition_approximation}, we can choose $\phi_{i,j} \in V$ and reindex them as $\phi_1,\phi_2,...,.\phi_n$ such that 
\begin{align}
    \|(h_1\comp h_2\comp ...\comp h_m)(x)-(\phi_1\comp \phi_2\comp ...\comp\phi_n)(x)\| \le \varepsilon/2, \quad
    \forall x \in \Omega.
\end{align}

(4) Combining (1) and (3), we finish the proof.

\end{proof}

\subsection{Proof of Theorem \ref{th:main-OP}}

\textit{\textbf{Theorem \ref{th:main-OP}.}
Let $\Omega \subset \mathbb{R}^d$ be a compact domain. Then, there is a finite set $V \subset \overline{\text{Diff}_0(\mathbb{R}^d)}$ such that the hypothesis space $\mathcal{H}_V$ in Eq.~\eqref{eq:H_v} has $C$-UAP for $\text{Diff}_0(\mathbb{R}^d)$.
}

\begin{proof}
    The Theorem \ref{th:main_V} provides constructive proof for the existence of $V$ in Eq.~\eqref{eq:V}.
\end{proof}

\textit{\textbf{Corollary \ref{th:main-C}.}
    Let $\Omega \subset \mathbb{R}^d$ be a compact domain, $d\ge 2$ and $p \in [1,+\infty)$. Then, there is a finite set $V \subset C(\mathbb{R}^d,\mathbb{R}^{d})$ such that the hypothesis space $\mathcal{H}_V$ in Eq.~\eqref{eq:H_v} has $L^p$-UAP for $C(\mathbb{R}^d,\mathbb{R}^{d})$.
}

\begin{proof}
    We can use the same $V$ in Theorem \ref{th:main-OP} as $V \subset \text{Diff}_0(\mathbb{R}^d) \subset C(\mathbb{R}^d,\mathbb{R}^{d})$.

    (1) Let $f \in C(\mathbb{R}^d,\mathbb{R}^{d}), d\ge 2$, then the result of \cite{Brenier2003Approximation} indicates that for any $\varepsilon >0$, there is a OP diffeomorphism $\Psi \in \text{Diff}_0(\mathbb{R}^d)$ such that
    \begin{align}
        \|f - \Psi\|_{L^p(\Omega)} < \varepsilon/2.
    \end{align}
    
    (2) The Theorem \ref{th:main-OP} indicates that, there is mapping $\Phi \in \mathcal{H}_V$ such that 
    \begin{align}
        \|\Psi(x) - \Phi(x)\| < \varepsilon' = \tfrac{\varepsilon}{2|\Omega|}, \quad \forall x \in \Omega.
    \end{align}

    (3) Combining (1) and (2), we have 
    \begin{align}
        \|f - \Phi\|_{L^p(\Omega)} < \varepsilon.
    \end{align}
    which finishes the proof.
\end{proof}

\section{Vocabulary for linear spaces}
\label{sec:V_linear}

Here we provide similar results for both the vector space and the linear mapping space. Note that linear mappings can be characterized as matrics and the construction here is much simpler than what we do in the main body of this paper for the continuous function space.

\begin{theorem}\label{th:Rd}
    There is a finite set $V_0 \subset \mathbb{R}^d$, such that for any vector $v^* \in \mathbb{R}^d$ and $\varepsilon>0$, there is a sequence, $v_{i_1}, v_{i_2},..., v_{i_n},$ in $V_0$, $n\in \mathbb{Z}_+$, such that 
    $$\| v_{i_1} + v_{i_2}+...+ v_{i_n} - v^*\| < \varepsilon.$$
\end{theorem}
\begin{proof}
    Directly employing Kronecker's Theorem \ref{th:Kronecker}, it is easy to see the following set satisfies the requirement,
    \begin{align}
        V_0 = \{ \lambda e_i | \lambda \in \{\pm1, \pm \sqrt{2}\}, i=1,2,...,d \},
    \end{align}
    where $e_i$ is the axis vector in the $i$-th coordinate.
\end{proof}

\begin{lemma}\label{th:R}
    Let $V_1 = \{0, \pm 1, 10^{\pm 1}, 10^{\pm \sqrt{2}}\}$, then for any number $\lambda \in \mathbb{R}$ and $\varepsilon>0$, there is a sequence, $v_{i_1}, v_{i_2},..., v_{i_n},$ in $V_1$, $n\in \mathbb{Z}_+$, such that 
    $$| v_{i_1} v_{i_2}... v_{i_n} - \lambda | < \varepsilon.$$
\end{lemma}
\begin{proof}
    It is enough to consider the case of $\lambda>0$. According to Theorem \ref{th:Rd} with $d=1$, we can finish the proof by approximating $v^*=\log_{10}{(\lambda)}$.
\end{proof}

\begin{theorem}
    There is a finite set $V_2 \subset \mathbb{R}^{d\times d}$, such that for any matrix $A^* \in \mathbb{R}^{d\times d}$ and $\varepsilon>0$, there is a sequence, $A_{i_1}, A_{i_2},..., A_{i_n},$ in $V_2$, $n\in \mathbb{Z}_+$, such that 
    $$\| A_{i_1} A_{i_2}... A_{i_n} - A^*\| < \varepsilon.$$
\end{theorem}
\begin{proof}
    For simplicity, we only consider the case of $d=2$ as the general cases can be proved in the same way. Since any singular matrix can be approximated by nonsingular matrixes, we only need to consider $A^*$ as a nonsingular matrix. In addition, every nonsingular matrix can be represented as a product of elementary matrices. Hence we can further assume $A^*$ to be an elementary matrix. Note that the elementary matrices are of the following,
    \begin{align*}
        \left(
        \begin{matrix} 
            \lambda &0\\ 
            0 & 1
        \end{matrix}
        \right),
        \left(
        \begin{matrix} 
            1 &0\\ 
            \lambda & 1
        \end{matrix}
        \right),
        \left(
        \begin{matrix} 
            0 &1\\ 
            1 & 0
        \end{matrix}
        \right), \lambda \neq 0.
    \end{align*}
    Therefore, we can finish the proof by considering the following set $V_2$,
    \begin{align}
        V_2 = \Big\{ 
            \left(
        \begin{matrix} 
            \lambda &0\\ 
            0 & 1
        \end{matrix}
        \right),
        \left(
        \begin{matrix} 
            1 &0\\ 
            1 & 1
        \end{matrix}
        \right),
        \left(
        \begin{matrix} 
            0 &1\\ 
            1 & 0
        \end{matrix}
        \right)
        \Big| \lambda \in \{\pm1, 10^{\pm 1}, 10^{\pm\sqrt{2}}\}\Big \}.
    \end{align}
    The validation of this $V_2$ can be verified by using Lemma~\ref{th:R} and the following relations,
    \begin{align*}
        \left(
        \begin{matrix} 
            1 &0\\ 
            \lambda & 1
        \end{matrix}
        \right)
        =
        \left(
        \begin{matrix} 
            1/\lambda &0\\ 
            0 & 1
        \end{matrix}
        \right)
        \left(
        \begin{matrix} 
            1 &0\\ 
            1 & 1
        \end{matrix}
        \right)
        \left(
        \begin{matrix} 
            \lambda &0\\ 
            0 & 1
        \end{matrix}
        \right)
        , \lambda \neq 0,
        \\
        \left(
        \begin{matrix} 
            \lambda_1 &0\\ 
            0 & 1
        \end{matrix}
        \right)
        \left(
        \begin{matrix} 
            \lambda_2 &0\\ 
            0 & 1
        \end{matrix}
        \right)
        =
        \left(
        \begin{matrix} 
            \lambda_1\lambda_2 &0\\ 
            0 & 1
        \end{matrix}
        \right)
        , \lambda_1, \lambda \in \mathbb{R}.
    \end{align*}

\end{proof}

\section{Proof of Theorem \ref{th:L_flowable}}

\textit{\textbf{Theorem \ref{th:L_flowable}.}
Let $\Omega=[0,1]^d, d\ge2,$ $p\in[1,\infty)$ be fixed. For any DFA $M$, there is a flow grammar $\mathcal{M}$ such that $L(M) = L(\mathcal{M})$.
}

\begin{proof}
    For any DFA $M = (Q, \Sigma, \delta, q_0, F)$ with $Q = \{q_0, . . . , q_{n-1}\}$, we will construct the flow grammar $\mathcal{M}$. Let's begin with the case that $F$ contains only one accept state. For this case, the construction is as follows.

    (1) Choosing a small number $\gamma \in (0,1/4)$ and $n$ different points $r_0, r_1, \cdots, r_{n-1}$ in $(2\gamma,1-2\gamma)^d$ such that $\|r_i - r_j\|_{1} > 3 \gamma$ for $i\neq j$. For each point $r_i$, define its correspoding cubic as $\Omega_i = \{r_i + x ~|~ \|x\|_1 < \gamma\}$. It is obvious that all cubes $\Omega_i$ are disjoint, and the $L^1$ distance between each pair of cubes is greater than $\gamma$.

    (2) Construct the embedding mappings. For each $s \in \Sigma$, we define its embedding mapping $\left\langle s \right\rangle$ as the following,
    \begin{align}
        \label{eq:embedding_s}
        \left\langle s \right\rangle: x \mapsto 
        \begin{cases*}
            x - r_i + r_j,  \quad x \in \Omega_i, q_j = \delta(q_i, s), \\
            \Delta_s(x), \quad x \notin \cup_{i=1}^n \Omega_i,
        \end{cases*}
    \end{align}
    where $\Delta_s$ is a funciton which makes $\left\langle s \right\rangle$ as a continuous mapping in $C(\Omega,\Omega)$. In other words, we extend the domain of definition from $\cup_{i=1}^n \Omega_i$ to the whole domain $\Omega$. This is guaranteed by the well-known Tietze extension theorem.

    (3) Construct the acceptance condition. Suppose the set of accept state is $F=\{q_{k}\}$, then the acceptance condition $A=(g(\cdot),\rho(\cdot),\epsilon)$ for $\mathcal{M}$ is designed as the following. Let $g(\cdot) = r_k$ be a constant function, $\rho(x) = \text{ReLU}(\gamma - \|x-r_0\|_1)$ be a piecewise linear function which vanishes outside $\Omega_0$, and $\epsilon = \min_{j \neq k} \frac{I_k + I_j}{2}$ where $I(j)$ is defined as the following,
    \begin{align}
        I_j := 
        \int_{\Omega} \rho(x) \|g(x) - (x - r_0 + r_j)\|^p dx
        =
        \int_{\Omega_0} \rho(x) \|r_k - (x - r_0 + r_j)\|^p dx, 
        \quad j = 0,1,\cdots, n-1.
    \end{align}
    Note that $I_j > I_k >0$ for any $j \neq k$ and hence the $\epsilon$ above is well defined. 

    (4) Verify $L(M) = L(\mathcal{M})$. The definition of the embedding in Eq.~\eqref{eq:embedding_s} indicates that $\left\langle s \right\rangle$ moves cubic $\Omega_i$ to $\Omega_j$ when $q_j = \delta(q_i, s)$. As a consequence, for a string $w = s_1 s_2\cdots s_m$ in $\Sigma^*$, we have $\left\langle w \right\rangle = 
    \left\langle s_1 \right\rangle
    \comp\cdots
    \comp
    \left\langle s_m \right\rangle$
    moves $\Omega_i$ to $\Omega_{j'}$ where $j'=i_m$ is the index satisfing $q_{i_0} = q_i$, $q_{i_{j+1}} = \delta(q_{i_j}, s_{j+1})$ for $j = 0, \cdots m - 1$. If $w$ is accept by $M$, \emph{i.e.} $w \in L(M)$, $i_0=0$ and $q_{i_m}=q_k$, then $\left\langle w \right\rangle$ moves $\Omega_0$ to $\Omega_k$. Consequently, we have
    \begin{align*}
        I(w) = 
        \int_{\Omega} \rho(x)
        \big\|
            g(x) - 
        \left\langle s_1 \right\rangle
        \comp\cdots
        \comp
        \left\langle s_m \right\rangle (x)
        \big\|^p  dx
        =
        I_{i_m}=
        I_k
        < \epsilon,
    \end{align*}
    which indicates that $w$ is accepted by $\mathcal{M}$, \emph{i.e.}, $w \in L(\mathcal{M})$. On the other hand, if $w \notin L(M)$, then $q_{i_m} \neq q_k$, $I(w)=I_{i_m} > \frac{I_k+I_{i_m}}{2} \ge \epsilon$ and $w \notin L(\mathcal{M})$. Therefore we finish the verification.

    Now we turn to the case that the set of accept state $F$ contains multiple states. 
    We assume $F=\{q_{k_1},...,q_{k_t}\} \subset Q$ contains $t$ states. For this case, we only need to modify the above construction slightly. In detail, we can replace the location of cubices $\Omega_0,\Omega_1,...,\Omega_n$ to satisfy an additional requirement: the minimal cube containing $\Omega_{k_1} \cup \cdots \cup \Omega_{k_t}$ has a $L^1$ distance larger than $\gamma$ to any other cube $\Omega_j$. This can be done by choosing a smaller number $\gamma$. Denote the minimal cube as $\hat \Omega_{k}$ which is centered at $\hat r_k$.
    The construction of the embedding mappings and the acceptance condition is the same as before except for modifying $g(\cdot)$ and $\epsilon$ to the following ones,
    \begin{align}
        g(\cdot)=\hat r_k 
        ,\quad 
        \epsilon = \frac{1}{2} \Big(\max_{k \in \{k_1,...,k_t\}} I_k + \min_{j \notin \{k_1,...,k_t\}}I_j \Big).
    \end{align}
    Note that the additional requirement on $\Omega_i$ ensures that $I_j > I_k >0$ for any $k \in \{k_1,...,k_t\}$ and $j \notin \{k_1,...,k_t\}$. As a consequence, the equality $L(M) = L(\mathcal{M})$ remains. 
\end{proof}

\newpage
\section{Potential insights for building models for natural language processing}
\label{sec:NLP}

In natural language processing (NLP), it is important to accurately depict the meaning of words and sentences. The well-known word vector embeddings provide a good baseline where words with similar semantics have similar word vectors. However, since static word vectors cannot describe the different semantics of polysemous words and the influence of context, people have developed dynamic word vector models and more complex large language models (LLM) such as BERT and GPT. However, how to interpret the pre-trained language models is difficult.

The implicit conclusion of our theorems is that if the meanings of sentences can be embedded as continuous functions (which is a much larger space than vector space), then we can express these meanings by the composition of a finite vocabulary of functions. This is the compositional flow-space model (CFSM) we proposed in Section \ref{sec:flow_grammar}. In CFSM, the semantics of polysemous words are judged by context, which is encoded in the input and output of the embedded functions.

Training such a CFSM from scratch is tricky and time-consuming. One alternative is to extract the embedded function directly from an LLM such as LLaMa, and then observe to what extent CFSM can restore the LLM's capabilities. Performing such an experiment is beyond the skill set of the authors.

Recently, the Mamba model has been getting a lot of attention. Its basic component is the state space model (SSM), which has a natural correlation with function composition. It's not hard to see that the single-layer (linear) SSM can be regarded as a function embedding of words. The Mamba model stacks multiple SSM layers and nonlinearity. It is worth thinking about the difference and relation between CFSM and the embedding in Mamba.

Of course, we should be wary of the fact that human's natural language is complex. Embedding words as functions is certainly a limited idea, but it's a good generalization compared to embedding words as vectors. 


\end{document}